\documentclass{article}
\usepackage[utf8]{inputenc}

\usepackage{geometry}
\usepackage[natbib, backend=biber, maxbibnames=99]{biblatex}
\AtEveryBibitem{
    \clearfield{doi}
    \clearfield{url}
    \clearfield{urlyear}
    \clearfield{urlmonth}
}

\usepackage[unicode=true,pdfusetitle,
 bookmarks=true,bookmarksnumbered=false,bookmarksopen=false,
 breaklinks=true,pdfborder={0 0 0},pdfborderstyle={},backref=false,colorlinks=true]{hyperref}
\hypersetup{citecolor=blue} 
 
\usepackage{graphicx}

\usepackage{amsmath}
\usepackage{amssymb}
\usepackage{amsthm}
\usepackage{amsfonts}
\usepackage{appendix}
\usepackage{booktabs}

\theoremstyle{plain}
\newtheorem{thm}{\protect\theoremname}
\theoremstyle{plain}
\newtheorem{lem}[thm]{\protect\lemmaname}
\theoremstyle{remark}

\theoremstyle{plain}

\theoremstyle{plain}

\theoremstyle{definition}

\theoremstyle{plain}
\newtheorem{prop}{\protect\propositionname}

\providecommand{\corollaryname}{Corollary}
\providecommand{\lemmaname}{Lemma}
\providecommand{\remarkname}{Remark}
\providecommand{\theoremname}{Theorem}
\providecommand{\conjecturename}{Conjecture}
\providecommand{\definitionname}{Definition}
\providecommand{\propositionname}{Proposition}

\usepackage{algorithm}
\usepackage{algpseudocode}
\usepackage{float}

\newcommand{\argmax}{\text{argmax}}

\newcommand{\R}{\mathbb{R}}

\newcommand{\E}{\mathbb{E}}
\newcommand{\Var}{\mathbb{V}}

\newcommand{\Phat}{\widehat{P}}
\newcommand{\Vhat}{\widehat{V}}
\newcommand{\Qhat}{\widehat{Q}}

\newcommand{\pistar}{\pi^\star}
\newcommand{\pihstar}{\widehat{\pi}^\star}

\newcommand{\one}{\mathbf{1}}

\newcommand{\pert}{\mathrm{p}}
\newcommand{\rpert}{\widetilde{r}}
\newcommand{\Rpert}{\widetilde{R}}
\newcommand{\ind}{\mathbb{I}}

\newcommand{\A}{\mathcal{A}}
\renewcommand{\S}{\mathcal{S}}

\newcommand{\gammared}{{\overline{\gamma}}}

\newcommand{\Otilde}{\widetilde{O}}

\newcommand{\tmix}{\tau_{\mathrm{unif}}}
\newcommand{\tstar}{\tau^\star}
\DeclareMathOperator*{\Clim}{\text{C-lim}}

\global\long\def\infnorm#1{\left\Vert #1\right\Vert _{\infty}}%
\global\long\def\infinfnorm#1{\left\Vert #1\right\Vert _{\infty \to \infty}}%
\global\long\def\onenorm#1{\left\Vert #1\right\Vert _{1}}%
\global\long\def\spannorm#1{\left\Vert #1\right\Vert _{\textnormal{span}}}%

\usepackage[dvipsnames]{xcolor}
\usepackage{todonotes}
\usepackage[showdeletions]{color-edits}
\addauthor[Matthew]{mz}{red}
\addauthor[Yudong]{yc}{blue}
\title{Span-Based Optimal Sample Complexity for Average Reward MDPs}
\usepackage{authblk}
\author{Matthew Zurek}
\author{Yudong Chen}
\affil{Department of Computer Sciences, University of Wisconsin-Madison\\\texttt{\{matthew.zurek,yudong.chen\}@wisc.edu}}

\date{}

\begin{document}
\maketitle

\begin{abstract}
    We study the sample complexity of learning an $\varepsilon$-optimal policy in an average-reward Markov decision process (MDP) under a generative model. We establish the complexity bound $\widetilde{O}\left(SA\frac{H}{\varepsilon^2} \right)$, where $H$ is the span of the bias function of the optimal policy and $SA$ is the cardinality of the state-action space. Our result is the first that is minimax optimal (up to log factors) in all parameters $S,A,H$ and $\varepsilon$, improving on existing work that either assumes uniformly bounded mixing times for all policies or has suboptimal dependence on the parameters.
    
    Our result is based on reducing the average-reward MDP to a discounted MDP. To establish the optimality of this reduction, we develop improved bounds for $\gamma$-discounted MDPs, showing that $\Otilde\left(SA\frac{H}{(1-\gamma)^2\varepsilon^2} \right)$ samples suffice to learn a $\varepsilon$-optimal policy in weakly communicating MDPs under the regime that $\gamma \geq 1 - \frac{1}{H}$, circumventing the well-known lower bound of $\widetilde{\Omega}\left(SA\frac{1}{(1-\gamma)^3\varepsilon^2} \right)$ for general $\gamma$-discounted MDPs. Our analysis develops upper bounds on certain instance-dependent variance parameters in terms of the span parameter. These bounds are tighter than those based on the mixing time or diameter of the MDP and may be of broader use.
\end{abstract}

\section{Introduction}

The paradigm of Reinforcement learning (RL) has recently received much attention and demonstrated remarkable successes in various sequential learning and decision-making problems. Empirical successes have motivated extensive theoretical study of RL algorithms and their fundamental limits.
The RL environment is commonly modeled as a Markov decision process (MDP), where the objective is to find a policy $\pi$ that maximizes the expected cumulative rewards. Different reward criteria have been considered, such as the finite horizon total reward $\E^\pi\left[ \sum_{t=0}^T R_t\right]$ and the  infinite horizon discounted reward criterion $\E^\pi\left[ \sum_{t=0}^\infty \gamma^t R_t\right]$ with a discount factor $\gamma < 1$. The finite horizon criterion only measures performance for $T$ steps, and the discounted criterion is dominated by the rewards received in the first $\frac{1}{1-\gamma}$ time steps. In many practical situations where the long-term performance of the policy $\pi$ is of interest, we may prefer to evaluate policies in terms of their long run average reward $\lim_{T \to \infty}\E^\pi\left[ \frac{1}{T}\sum_{t=0}^{T-1} R_t\right]$.

A foundational theoretical problem in RL is the sample complexity for learning a near-optimal policy when we have access to a generative model of the MDP \citep{kearns_finite-sample_1998}, meaning the ability to obtain independent samples of the next state given any initial state and action. For the finite horizon and discounted reward criteria, the sample complexity of this problem has been thoroughly studied and well understood (e.g., \cite{azar_sample_2012, gheshlaghi_azar_minimax_2013, sidford_near-optimal_2018, wainwright_variance-reduced_2019, agarwal_model-based_2020, li_breaking_2020}). However, despite significant effort (reviewed in Section~\ref{sec:related}), the sample complexity of  the average reward setting is unresolved in existing literature.

In this paper, we resolve the sample complexity of Average Reward MDPs in terms of $H := \spannorm{h^\star}$,
the span of the bias (a.k.a.\ relative value function) of the optimal policy. We show that
\[
\Otilde\left(SA\frac{H}{\varepsilon^2}\right)
\]
samples suffice to find an $\varepsilon$-optimal policy of an MDP with $S$ states and $A$ actions, under the assumption that the MDP is weakly communicating. This bound, presented formally in Section \ref{sec:results} as Theorem \ref{thm:main_theorem}, is the first that matches the minimax lower bound $\widetilde{\Omega}\left( SAH/\varepsilon^2\right)$ up to log factors. 

To establish the above result, we adopt the the reduction-to-discounted-MDP approach pioneered by \cite{jin_towards_2021} and \cite{wang_near_2022}, and develop an improved bound for the complexity of certain discounted MDPs. This result for discounted MDPs, which appears as our Theorem \ref{thm:DMDP_bound}, shows that
\[
\Otilde\left(SA\frac{H}{(1-\gamma)^2\varepsilon^2}\right)
\]
samples are sufficient for finding an $\varepsilon$-optimal policy in the discounted setting, under the mild additional assumptions of a weakly communicating MDP and a sufficiently large effective horizon so that $\gamma \geq 1 - \frac{1}{H}$. In this setting, the above result circumvents the well-known minimax lower bound of $\widetilde{\Omega}\left(SA\frac{1}{(1-\gamma)^3\varepsilon^2} \right)$ for the sample complexity of discounted MDPs. 

Note that the span $H$ is always upper bounded by the diameter $D$ of the MDP, by the mixing time $\tstar$ of any optimal policy, and in turn by a uniform upper bound $\tmix$ on the mixing times of all policies; see Section~\ref{sec:formulation} for formal definitions. Consequently,  our two results above immediately imply sample complexity bounds in terms of  $D$, $\tstar$ and $\tmix$, though our span-based bounds are often substantially stronger. In particular, our results do \emph{not} require uniform mixing (i.e., $\tmix <\infty$) or strongly communicating (i.e., $D<\infty$), which are common but restrictive assumptions used in the average-reward MDP literature. 

To prove the above results, we make use of the algorithm from \cite{li_breaking_2020} for discounted MDPs, but provide a specialized analysis using the span parameter $H$.
By improving upon the analysis used in previous reductions from average-reward to discounted MDPs, we believe our work sheds greater light on the relationship between these two problems. Our approach is discussed further in Subsection \ref{sec:our_approach}, and a more detailed proof sketch is provided in Section \ref{sec:proof_sketch}.

\subsection{Related Work}
\label{sec:related}

Below we first discuss related work and compare to our result.

\subsubsection*{Average-reward MDPs} 

We summarize in Table \ref{table:AMDPs} the relevant literature on the sample complexity of average reward MDPs under a generative model. In all references, the objective is to find an $\varepsilon$-optimal policy $\pihstar$.

\begin{table}[H]
{\renewcommand{\arraystretch}{1.4} 
\centering
\begin{tabular}{|c|c|c|c|}
\hline
Method & Sample Complexity & Reference & Comments \\ \hline \hline
Primal-Dual SMD & $\Otilde\left( SA \frac{\tmix^2}{\varepsilon^2}\right)$ & \cite{jin_efficiently_2020} & requires uniform mixing \\ \hline
Reduction to Discounted MDP & $\Otilde\left( SA \frac{\tmix}{\varepsilon^3}\right)$ & \cite{jin_towards_2021} & requires uniform mixing \\ \hline
Policy Mirror Descent & $\Otilde\left( SA \frac{\tmix^3}{\varepsilon^2}\right)$ & \cite{li_stochastic_2022} & requires uniform mixing \\ \hline
Reduction to Discounted MDP & $\Otilde\left(SA \frac{\tmix}{\varepsilon^2} \right)$ & \cite{wang_optimal_2023-1} & requires uniform mixing \\ \hline \hline 
Reduction to Discounted MDP & $\Otilde\left(SA\frac{H}{\varepsilon^3} \right)$ & \cite{wang_near_2022} &  \\ \hline 
Refined Q-Learning & $\Otilde\left(SA\frac{H^2}{\varepsilon^2} \right)$ & \cite{zhang_sharper_2023} &  \\ \hline 
Reduction to Discounted MDP & $\Otilde\left(SA\frac{H}{\varepsilon^2} \right)$ & Our Theorem \ref{thm:main_theorem} &  \\ \hline \hline
Lower Bound & $\widetilde{\Omega}\left( SA\frac{\tmix}{\varepsilon^2}\right)$ & \cite{jin_towards_2021} & implies $\widetilde{\Omega}\left( SA\frac{H}{\varepsilon^2}\right)$\\ \hline
Lower Bound & $\widetilde{\Omega}\left( SA\frac{D}{\varepsilon^2}\right)$ & \cite{wang_near_2022} & implies $\widetilde{\Omega}\left( SA\frac{H}{\varepsilon^2}\right)$ \\ \hline
\end{tabular}
\caption{\textbf{Comparison of algorithms and sample complexity bounds} for finding an $\varepsilon$-optimal policy of average reward MDPs with $S$ states and $A$ actions. Here $H:= \spannorm{h^\star}$ is the span of the bias of an optimal policy,  $\tmix$ is a uniform upper bound on the mixing times of all policies, and $D$ is the diameter of the MDP, with the relationships $H\le 8\tmix$ and $H\le D.$}
\label{table:AMDPs}
}
\end{table}

Various parameters have been used to characterize the sample complexity of average reward MDPs, including the diameter $D$ of the MDP, the uniform mixing time bound $\tmix$ for all policies, and the span $H$ of the optimal bias; formal definitions and more comparisons are provided in Section \ref{sec:formulation}. All sample complexity upper bounds involving $\tmix$ require the strong assumption that \emph{all} stationary policies have finite mixing times. Otherwise, we have $\tmix=\infty$ by definition, which occurs when, for example, there exists any policy which induces a periodic Markov chain. The situation $D = \infty$ is also possible, while $H$ is always finite in weakly communicating MDPs with finite state-action spaces. As shown in \cite{wang_near_2022}, there is generally no relationship between $D$ and $\tmix$; they can each be arbitrarily larger than the other. On the other hand, it has been shown that $H \leq D$ \citep{bartlett_regal_2012} and that $H \leq 8 \tmix$ \citep{wang_near_2022}, so either minimax lower bound in Table \ref{table:AMDPs} implies a lower bound $\widetilde{\Omega}\left( SA\frac{H}{\varepsilon^2}\right)$ and thus the minimax optimality of our Theorem \ref{thm:main_theorem}.

The work \cite{jin_towards_2021} was the first to develop an algorithm based on reduction to a discounted MDP with a discount factor of $\gamma = 1 - \frac{\varepsilon}{\tmix}$. Their argument was improved in \cite{wang_near_2022}, which removed the uniform mixing assumption and used a smaller discount factor $\gamma = 1 - \frac{\varepsilon}{H}$. After analyzing the reductions, both \cite{jin_towards_2021} and \cite{wang_near_2022} then solved the discounted MDPs by appealing to the algorithm from \cite{li_breaking_2020}. To the best of our knowledge, the algorithm of \cite{li_breaking_2020} is the only known algorithm for discounted MDPs which could work with either reduction, as the reductions each require a $\frac{\varepsilon}{1-\gamma}$-optimal policy from the discounted MDP, and other known algorithms for discounted MDPs do not permit such large suboptimality levels. (We discuss algorithms for discounted MDPs in more detail below.) Other algorithms for average-reward MDPs are considered in \cite{jin_towards_2021,li_stochastic_2022,zhang_sharper_2023}. The above results fall short of matching the minimax lower bounds.

While preparing this manuscript, we became aware of the independent work \cite{wang_optimal_2023-1}, who considered the uniform mixing setting and obtained a minimax optimal sample complexity $\Otilde\left(SA \frac{\tmix}{\varepsilon^2} \right)$ in terms of $\tmix$. Although developed independently, their work and ours have several similarities. We both observe that it is possible to improve the variance analysis of the algorithm from \cite{li_breaking_2020} to get a superior complexity for discounted MDPs in regimes relevant to average-to-discounted reduction, which involves a large value of $\gamma$ or equivalently a large effective horizon $1/(1-\gamma)$. They accomplish the improvement by leveraging the uniform mixing assumption, whereas we make use of the low span of the optimal policy. Note that $H \leq 8 \tmix$ holds in general and there exist MDPs with $H \ll \tmix=\infty$, so our Theorem \ref{thm:main_theorem} is strictly stronger than  the result of \cite{wang_optimal_2023-1}.

\subsubsection*{Discounted MDPs} 

We discuss a subset of results for discounted MDPs in the generative setting. Several works \cite{sidford_near-optimal_2018, wainwright_variance-reduced_2019, agarwal_model-based_2020, li_breaking_2020} obtain the minimax optimal sample complexity of $\Otilde\left( SA\frac{1}{(1-\gamma)^3\varepsilon^2}\right)$, but only \cite{li_breaking_2020} are able to show this bound for the full range of $\varepsilon \in (0, \frac{1}{1-\gamma}]$. (Note that now the goal is to find a $\varepsilon$-optimal policy in the discounted MDP, so $\varepsilon$ has a different meaning than in the average-reward setting; more detail is given in Section \ref{sec:formulation}.) As mentioned above, the reduction from average reward to discounted MDPs requires a large $\varepsilon$, making it unsurprising that all of \cite{jin_towards_2021, wang_near_2022, wang_optimal_2023-1} as well as our Algorithm \ref{alg:DMDP_alg} essentially use their algorithm. The matching lower bound is established in \cite{sidford_near-optimal_2018} making use of the techniques from \cite{gheshlaghi_azar_minimax_2013}.

As mentioned earlier, both we and the authors of \cite{wang_optimal_2023-1, wang_optimal_2023} independently observed that the $\widetilde{\Omega}\left( SA\frac{1}{(1-\gamma)^3\varepsilon^2}\right)$ sample complexity lower bound can be circumvented in the settings that arise under the average-to-discounted reductions considered \cite{jin_towards_2021} and \cite{wang_near_2022}. The authors of \cite{wang_optimal_2023-1, wang_optimal_2023} assume uniform mixing and $\frac{1}{1-\gamma} \geq \tmix$ and obtain a discounted MDP sample complexity of $\Otilde\left(SA \frac{\tmix}{(1-\gamma)^2 \varepsilon^2} \right)$, first in \cite{wang_optimal_2023} by modifying the algorithm of \cite{wainwright_variance-reduced_2019}, and then in \cite{wang_optimal_2023-1} under a wider range of $\varepsilon$ by instead modifying the analysis of \cite{li_breaking_2020}. \cite{wang_optimal_2023} also proves a matching lower bound. Our Theorem \ref{thm:DMDP_bound} for discounted MDPs attains a sample complexity of $\Otilde\left(SA \frac{H}{(1-\gamma)^2 \varepsilon^2} \right)$ under the requirement $\frac{1}{1-\gamma} \geq H$, assuming only that the MDP is weakly communicating. Again, in light of the relationship that $H \leq 8 \tmix$, our sample complexity as well as our condition on $\gamma$ dominate their results (ignoring constants), and their lower bound also establishes the optimality of our Theorem \ref{thm:DMDP_bound}.

\subsection{Our Approach}
\label{sec:our_approach}

Now we give a high-level description of our techniques. We apply a simple reduction-based approach to solving the average-reward MDP, reducing this problem to solving a discounted MDP with the same reward and transition kernel and a judicious choice of discount factor $\gamma$. This method was first developed by \cite{jin_towards_2021}, who assumed bounded mixing times for all policies, and later improved by \cite{wang_near_2022} to only require a bound on $\spannorm{h^\star}$. We record their result as Lemma \ref{lem:DMDP_reduction}.

By using the reduction from \cite{wang_near_2022}, to obtain an $O(\varepsilon)$-optimal policy for the average-reward MDP, it suffices to obtain an $H$-optimal policy for the discounted MDP with discount factor $\gamma = 1-\frac{\varepsilon}{H}$. If we had an algorithm for discounted MDPs which required $\Otilde\left(SA \frac{H}{(1-\gamma)^2\varepsilon^2} \right)$ samples to obtain $\varepsilon$-accuracy, then setting $\varepsilon=H$ and $\gamma = 1-\frac{\varepsilon}{H}$would yield an optimal sample complexity of
\[\Otilde\left( SA\frac{H}{(1-\gamma)^2\varepsilon^2}\right)=\Otilde\left( SA\frac{H}{\left(\frac{\varepsilon}{H}\right)^2 H^2}\right) = \Otilde\left( SA\frac{H}{\varepsilon^2}\right)\]
for the average-reward MDP.
Note that \cite{wang_near_2022} show a similar calculation, but they use a discounted MDP oracle complexity of $\Otilde\left(SA \frac{1}{(1-\gamma)^3\varepsilon^2} \right)$, leading to an additional factor of $\frac{1}{\varepsilon}$.

As a first approximation, our key insight to improving the sample complexity for discounted MDPs is the observation that the $\frac{1}{(1-\gamma)^3}$ factor arises from an upper bound on a certain instance-dependent variance parameter which, in the setting of weakly-communicating MDPs with $\spannorm{h^\star} \leq H$ and with large discount factors ($\frac{1}{1-\gamma} \geq H$), can actually bounded by $\frac{H}{(1-\gamma)^2}$. Here we highlight some essential ideas for proving this improved variance bound and defer a more detailed description to the proof sketch in Section \ref{sec:proof_sketch}. To relate the variance parameters to $\spannorm{h^\star}$, the key ingredient is a multistep version of the variance Bellman equation (\cite[Theorem 1]{sobel_variance_1982}) with a carefully balanced number of steps. In addition, the variance parameters must be controlled both for the true optimal policy as well as the policy $\pihstar$ that is optimal for an empirical MDP constructed from samples. This is accomplished by using the multistep variance Bellman relation to bound the variance in terms of the suboptimality of $\pihstar$, yielding a recursive bound on the suboptimality of $\pihstar$.

\section{Problem Setup}
\label{sec:formulation}

A Markov decision process (MDP) is given by a tuple $(\S, \A, P, r)$, where $\S$ is the finite set of states, $\A$ is the finite set of actions, $P : \S \times \A \to \Delta(\S)$ is the transition kernel with $\Delta(\S)$ denoting the probability simplex over $\S$, and $r : \S \times \A \to [0,1]$ is the reward function. Let $S := |\S|$ and $A := |\A|$ denote the cardinality of the state and action spaces, respectively. In this paper, we only consider stationary Markovian policies of the form $\pi : \S \to \Delta(\A)$. For any initial state $s_0 \in \S$ and policy $\pi$, we let $\E^\pi_{s_0}$ denote the expectation with respect to the probability distribution over trajectories $(S_0, A_0, S_1, A_1, \dots)$ where $S_0 = s_0$, $A_t \sim \pi(S_t)$, and $S_{t+1} \sim P(\cdot \mid S_t, A_t)$. Equivalently, this is the expectation with respect to the Markov chain induced by $\pi$ starting in state $s_0$,  with the transition probability matrix $P_\pi$ given by  $\left(P_\pi\right)_{s,s'} := \sum_{a \in \A} \pi(a \mid s) P(s' \mid s, a)$. We also define $(r_\pi)_{s} := \sum_{a \in \A} \pi(a \mid s) r(s, a)$. For any $s \in \S$ and any bounded function $X$ of the trajectory, we define the variance $\Var^\pi_s\left[X\right] := \E^\pi_s \left(X - \E^\pi_s\left[X \right] \right)^2$, with its vector version $\Var^\pi\left[X\right] \in \R^\S$ given by  $\left(\Var^\pi\left[X\right] \right)_s = \Var^\pi_s\left[X\right]$.
For $s \in \S$, let $e_s \in \R^{\S}$ be the vector that is all $0$ except for a $1$ in entry $s$. Let $\one \in \R^{\S}$ be the all-one vector. For each $v \in \R^{\S} $, define the span semi-norm  \[\spannorm{v} := \max_{s \in \S} v(s) - \min_{s \in \S} v(s).\]

Throughout this paper We assume the MDP is weakly-communicating \cite{puterman_markov_2014}, meaning that the states can be partitioned into two subsets such that in the first subset, all states are transient under any stationary policy, and in the second subset, any state is reachable from any other state under some stationary policy.

A discounted MDP is a tuple $(\S, \A, P, r, \gamma)$, where $\gamma \in (0,1)$ is the discount factor. For any stationary policy $\pi$, the (discounted) value function $V^\pi_\gamma : \S \to [0, \infty)$ is defined, for each $s \in \S$, as
\begin{align}
    V^\pi_\gamma(s) := \E^\pi_s \left[\sum_{t=0}^\infty \gamma^t R_t \right] \label{eq:value_fn_defn}
\end{align}
where $R_t = r(S_t, A_t)$ is the reward received at time $t$. It is well-known that there exists an optimal policy $\pistar_\gamma$ that is deterministic and satisfies $V_\gamma^{\pistar_\gamma}(s) = V_\gamma^\star(s) := \sup_{\pi} V_\gamma^\pi(s)$ for all $s \in \S$ \cite{puterman_markov_2014}. In discounted MDPs the goal is to compute an $\varepsilon$-optimal policy,  which we define as a policy $\pi$ satisfying $\infnorm{V_\gamma^\pi - V_\gamma^\star} \leq \varepsilon$. We define one more variance parameter $\Var_{P_{\pi}}\left[V_\gamma^{\pi} \right] \in \R^{\S}$, specific to a policy $\pi$, by
\[
\left(\Var_{P_{\pi}} \left[V_\gamma^{\pi} \right]\right)_s := \sum_{s' \in \S} \left(P_{\pi}\right)_{s, s'} \left(V_\gamma^{\pi}(s') - \sum_{s''}\left(P_{\pi}\right)_{s, s''}V_\gamma^{\pi}(s'')\right)^2.
\]

In an MDP $(\S, \A, P, r)$ under the average-reward criterion, the average reward per stage or the \emph{gain} of a policy $\pi$ starting from state $s$ is defined as
\begin{align*}
    \rho^\pi(s)  := \lim_{T \to \infty} \frac{1}{T} \E_s^\pi \left[\sum_{t=0}^{T-1} R_t \right].
\end{align*}
When the MDP is weakly communicating, there exist $q^\star : \S \times \A \to \R$ and $\rho^\star \in \R$ that satisfy the Bellman/Poisson equation
\begin{align*}
    \rho^\star + q^\star(s,a) = r(s,a) + \sum_{s' \in \S} P(s' \mid s, a) \max_{a' \in \A} q^\star(s',a'), \quad \forall s, a \in \S \times \A.
\end{align*}
Furthermore, $\rho^\star$ is unique and $q^\star$ is unique up to additive shifts by $\alpha \one$ for $\alpha \in \R$. The optimal \emph{bias} function $h^\star:\S \to \R$ is defined by $h^\star(s) := \max_{a' \in \A} q^\star(s,a')$. Note that $h^\star$ is only unique up to additive shifts, but the span $H:=\spannorm{h^\star}$ is uniquely defined.
Defining the policy $\pistar(s) := \argmax_{a' \in \A} q^\star(s, a'),\forall s\in\S$, we have that $\rho^{\pistar}(s) \geq \rho^\pi$ for any stationary policy $\pi$. Also $\rho^{\pistar}(s) = \rho^\star$ for all $s$, so we slightly abuse notation and define $\rho^\star = \rho^{\pistar}$ (a vector). While we only need to consider the bias of the optimal policy $\pistar$ in this paper, we record the following standard definition for the bias function of any stationary policy $\pi$:
\begin{align*}
    h^\pi(s) := \Clim_{T \to \infty} \E_s^\pi \left[\sum_{t=0}^{T-1} R_t - T\rho^\pi(s) \right],
\end{align*}
where $\Clim$ denotes the Cesaro limit. Under this definition, we have $h^{\pistar} = h^\star$ (where the equality is modulo $+\alpha\one$). When the Markov chain induced by $P_\pi$ is aperiodic, $\Clim$ can be replaced with the usual limit. In the average reward setting, our goal is find an $\varepsilon$-optimal policy, defined as a policy $\pi$ such that $\infnorm{\rho^\star - \rho^\pi} \leq \varepsilon$.

We now define several other parameters that have been used to characterize sample complexity in average reward MDPs. For each $s\in \S$, let $\eta_{s}$ denote the hitting time of $s$. The diameter $D$ is defined as
\[
D := \max_{s_1 \neq s_2} \inf_{\pi\in \Pi} \E^\pi_{s_1} \left[\eta_{s_2}\right],
\]
where $\Pi$ is the set of all deterministic stationary policies. 
For each stationary policy $\pi$, if the Markov chain induced by $P_\pi$ has a unique stationary distribution $\nu_\pi$, we define the mixing time of $\pi$ as
\begin{align*}
    \tau_\pi := \inf \left\{t \geq 1 : \max_{s \in \S} \onenorm{e_s^\top \left(P_\pi \right)^t - \nu^\top_\pi} \leq \frac{1}{2} \right\}.
\end{align*}
If all policies $\pi\in\Pi$ satisfy this assumption, we define the uniform mixing time $\tmix := \sup_{\pi\in\Pi} \tau_\pi$.
Let $\Pi^\star$ denote the set of all deterministic policies $\pi$ such that $\pi(s) \in \argmax_a q^\star(s,a)$ and that $\pi$ has a unique stationary distribution. We can then define the optimal mixing time
\begin{align*}
    \tstar := \inf_{\pi \in \Pi^\star} \tau_\pi.
\end{align*}
By definition we trivially have $\tstar \leq \tmix$. Note that $D$ and $\tmix$ are generally incomparable \cite{wang_near_2022}, while we always have $H \leq D$ \cite{bartlett_regal_2012} and $H \leq 8 \tmix$ \cite{wang_near_2022}. In fact, we also have $H \leq 8\tstar$, as shown in Proposition \ref{thm:optimal_tmix_bound}. It is possible for $\tmix = \infty$, for instance if there are any policies which induce periodic Markov chains. Also, if there are any states which are transient under all policies, then $D = \infty$. However, $H$ is finite in any weakly communicating MDP with $S, A < \infty$.

We assume access to a generative model \cite{kearns_finite-sample_1998}, also known as a simulator. This means we can obtain independent samples from $P(\cdot \mid s, a)$ for any given $s \in \S, a \in \A$, but $P$ itself is unknown. We assume that the reward function $r$ is deterministic and known, which is a standard assumption in generative settings (e.g., \cite{agarwal_model-based_2020, li_breaking_2020}) since otherwise estimating the mean rewards is relatively easy. Specifically, to learn an $\varepsilon$-optimal policy for the discounted MDP in our setting, we would need to estimate each entry of $r$ to accuracy $O((1-\gamma)\varepsilon)$, which requires a lower order number of samples $\Otilde\left(\frac{SA}{(1-\gamma)^2\varepsilon^2}\right)$. For this reason we assume (as in \cite{wang_near_2022}) that $H \geq 1$. 

Using samples from the generative model, our Algorithm \ref{alg:DMDP_alg} constructs an empirical transition kernel $\Phat$. For a policy $\pi$ and state $s$, we use $\Vhat^\pi_\gamma(s)$ to denote a value function computed with respect to the Markov chain with transition matrix $\Phat_\pi$ (as opposed to $P_\pi$). Our Algorithm \ref{alg:DMDP_alg} also utilizes a perturbed reward function $\rpert$, and we use the notation $V^\pi_{\gamma, \pert}(s)$ to denote a value function computed using this reward (and $P_\pi$); more concretely, we replace $R_t$ with $\Rpert_t = \rpert(S_t, A_t)$ in equation~\eqref{eq:value_fn_defn}. We use the notation $\Vhat^\pi_{\gamma, \pert}$ when using $\Phat$ and $\rpert$ simultaneously. Since we only ever consider the state space $\S$ and action space $\A$,  we often omit reference to them when describing MDP parameters and simply write $(P,r)$ or $(P,r,\gamma)$.

\section{Main Results}
\label{sec:results}

Our approach is based on reducing the average-reward problem to a discounted problem. We first present our algorithm and guarantees for the discounted MDP setting.

\begin{algorithm}[H]
\caption{Perturbed Empirical Model-Based Planning} \label{alg:DMDP_alg}
\begin{algorithmic}[1]
\Require Sample size per state-action pair $n$, target accuracy $\varepsilon \in (0, H]$, discount factor $\gamma$
\For{each state-action pair $(s,a) \in \S \times \A$}
\State Collect $n$ samples $S^1_{s,a}, \dots, S^n_{s,a}$ from $P(\cdot \mid s,a)$
\State Form the empirical transition kernel $\Phat(s' \mid s, a) = \frac{1}{n}\sum_{i=1}^n \ind\{S^i_{s,a} = s'\}$, for all $s' \in \S$
\EndFor
\State Set perturbation level $\xi = \frac{(1-\gamma)\varepsilon}{6}$
\State Form perturbed reward $\rpert = r + Z$ where $Z(s,a) \stackrel{\text{i.i.d.}}{\sim} \text{Unif}(0, \xi)$
\State Compute a policy $\pihstar_{\gamma, \pert}$ which is optimal for the perturbed empirical discounted MDP $(\Phat, \rpert, \gamma)$
\State \Return $\pihstar_{\gamma, \pert}$
\end{algorithmic}
\end{algorithm}

As discussed in Subsection \ref{sec:related}, our algorithm of choice, Algorithm~\ref{alg:DMDP_alg}, is essentially the same as the one presented in \cite{li_breaking_2020}, with a slightly different perturbation level $\xi$. In Algorithm \ref{alg:DMDP_alg} we construct an empirical transition kernel $\Phat$ using samples from the generative model, and then solve the resulting empirical (perturbed) MDP $(\Phat, \rpert, \gamma)$, for which any MDP solver can be used. We remark in passing that $SA$-by-$S$ transition matrix $\Phat$ has at most $nSA$ nonzero entries.

Our Theorem \ref{thm:DMDP_bound} provides an improved sample complexity bound for Algorithm~\ref{alg:DMDP_alg} under the setting that the MDP is weakly communicating and the span satisfies $H \leq \frac{1}{1-\gamma}$.

\begin{thm}[Sample Complexity of Discounted MDP]
    \label{thm:DMDP_bound}
    Suppose $H \leq \frac{1}{1-\gamma}$ and $\varepsilon \leq H$. There exists a constant $C_2 > 0$ such that, for any $\delta \in (0,1)$, if $n \geq C_2\frac{H}{(1-\gamma)^2\varepsilon^2} \log \left( \frac{S A}{(1-\gamma)\delta \varepsilon}\right)$, then with probability at least $1-\delta$, the policy $\pihstar_{\gamma, \pert}$ output by Algorithm \ref{alg:DMDP_alg} satisfies
    \begin{align*}
        \infnorm{V^\star - V^{\pihstar_{\gamma, \pert}}} \leq \varepsilon.
    \end{align*}
\end{thm}

Since we observe $n$ samples for each state-action pair, Theorem~\ref{thm:DMDP_bound} shows that $\Otilde \left(\frac{HSA}{(1-\gamma)^2\varepsilon^2} \right)$ total samples suffice to learn an $\varepsilon$-optimal policy.
The assumption of the theorem requires the span $H = \spannorm{h^\star}$ to be no larger than the effective horizon $\frac{1}{1-\gamma}$, or equivalently for the discount factor $\gamma$ to be at least $1 - \frac{1}{H}$. This assumption holds in many situations, as can be seen by using the relationships $H \leq D$, $H \leq 8\tmix$, or $H \leq 8\tstar$. For instance, the requirement $D \lesssim \frac{1}{1-\gamma}$  means that it is possible to reach any other state within the effective horizon, and the requirement $\tstar \lesssim \frac{1}{1-\gamma}$ means that an optimal policy (of the underlying average reward MDP) has enough time to mix within the effective horizon. 
On the other hand, in the regime with $H>\frac{1}{1-\gamma}$, the known minimax optimal sample complexity for $\gamma$-discounted MDPs of $\Otilde\left( \frac{SA}{(1-\gamma)^3 \varepsilon^2}\right)$, also achieved by Algorithm~\ref{alg:DMDP_alg}, is superior. In this regime, the discounting effectively truncates the MDP at a short horizon $\frac{1}{1-\gamma}$ before the long-run behavior of the optimal policy (as captured by $H$) kicks in.\\

Now we present our main result for the average-reward problem. Our Algorithm~\ref{alg:AMDP_alg} reduces the problem to $\gammared$-discounted MDP with $\gammared = 1- \frac{\varepsilon}{12 H}$ and then calls Algorithm~\ref{alg:DMDP_alg} with target accuracy $H$. 

\begin{algorithm}[H]
\caption{Discounted MDP Reduction} \label{alg:AMDP_alg}
\begin{algorithmic}[1]
\Require Sample size per state-action pair $n$, target accuracy $\varepsilon \in (0, 1]$, $H = \spannorm{h^\star}$
\State Set $\gammared = 1- \frac{\varepsilon}{12 H}$
\State Obtain $\pihstar$ from Algorithm \ref{alg:DMDP_alg} with sample size per state-action pair $n$, accuracy $H$, discount $\gammared$
\State \Return $\pihstar$
\end{algorithmic}
\end{algorithm}

We have the following sample complexity bound for Algorithm~\ref{alg:AMDP_alg}.
\begin{thm}[Sample Complexity of Average-reward MDP]
\label{thm:main_theorem}
    There exists a constant $C_1>0$ such that for any $\delta, \varepsilon \in (0,1)$, if $n \geq C_1 \frac{H}{\varepsilon^2} \log \left( \frac{S A}{\delta \varepsilon}\right)$, then with probability at least $1-\delta$, the policy $\pihstar$ output by Algorithm~\ref{alg:AMDP_alg} satisfies the elementwise inequality
    \begin{align*}
        \rho^\star - \rho^{\pihstar} \leq \varepsilon \one.
    \end{align*}
\end{thm}

Again, since we observe $n$ samples for each state-action pair, this result shows that $\Otilde \left( \frac{HSA}{\varepsilon^2} \right)$ total samples suffice to learn an $\varepsilon$-optimal policy for the average reward MDP. This bound matches the minimax lower bound in~\cite{wang_near_2022} and is superior to existing results (see Table~\ref{table:AMDPs}). We note that since Algorithm \ref{alg:DMDP_alg} and its proof work so long as $H$ is any upper bound of $\spannorm{h^\star}$, Algorithm~\ref{alg:AMDP_alg} also only needs an upper bound for $\spannorm{h^\star}$. \\

We have mentioned the relationship $H\le 8\tstar$, for which we recall that $\tstar$ is the smallest mixing time of any optimal policy for the average reward MDP (see Section \ref{sec:formulation}). We formally state this result below. Together with Theorems~\ref{thm:DMDP_bound} and~\ref{thm:main_theorem}, this result immediately implies sample complexity bounds in terms of $\tstar$ for discounted and average-reward MDPs.

\begin{prop}
\label{thm:optimal_tmix_bound}
    In any weakly communicating MDP, $H \leq 8\tstar$. Consequently,  Theorems \ref{thm:DMDP_bound} and \ref{thm:main_theorem} hold with $H$ replaced by $8 \tstar$ (including within Algorithms \ref{alg:DMDP_alg} and \ref{alg:AMDP_alg}).
\end{prop}

Note that $\tstar$ is a property of only the optimal policies, and the above result does not require all policies to have finite mixing times. The proof of Proposition~\ref{thm:optimal_tmix_bound} follows similar arguments as in~\cite[Proposition 10]{wang_near_2022}, which concerns $\tmix$.

\section{Outline of Analysis}
\label{sec:proof_sketch}

As discussed in Subsection \ref{sec:our_approach}, our main Theorem \ref{thm:main_theorem} on the complexity of average reward MDPs follows from using the reduction of \cite{wang_near_2022} with the improved discounted MDP sample complexity from Theorem \ref{thm:DMDP_bound}, and the key to obtaining this improved complexity is a careful analysis of certain instance-dependent variance parameters.
More concretely, the arguments in \cite{li_breaking_2020} (repackaged in our Lemma \ref{lem:DMDP_error_bounds}) demonstrate that it would suffice to bound the variance parameters $\infnorm{(I - \gamma P_{\pistar_\gamma})^{-1} \sqrt{\Var_{P_{\pistar_\gamma}} \left[V_\gamma^{\pistar_\gamma} \right]}}$ and $\infnorm{(I - \gamma P_{\pihstar_{\gamma, \pert}})^{-1} \sqrt{\Var_{P_{\pihstar_{\gamma, \pert}}} \left[V_{\gamma, \pert}^{\pihstar_{\gamma, \pert}} \right]}}$ each by $O\left(\sqrt{\frac{H}{(1-\gamma)^2}}\right)$. The first of these variance parameters, which pertains to the optimal policy $\pistar_\gamma$ of the discounted MDP, is more straightforward to bound. By Lemma \ref{lem:var_params_relationship}, it suffices to bound $\infnorm{\Var^{\pistar_{\gamma}}\left[ \sum_{t=0}^{\infty}\gamma^t R_t \right]} \leq O\left(\frac{H}{1-\gamma} \right)$ (as opposed to the obvious bound of $\frac{1}{(1-\gamma)^2}$, which follows from the fact that the total reward in a trajectory is bounded within $[0, \frac{1}{1-\gamma}]$ and is used in, e.g., \cite{li_breaking_2020} to obtain a sample complexity scaling with $\frac{1}{(1-\gamma)^3}$). The first step is to decompose $\Var^{\pistar_{\gamma}}\left[ \sum_{t=0}^{\infty}\gamma^t R_t \right]$ recursively like
\[
    \Var^{\pistar_{\gamma}}\left[ \sum_{t=0}^{\infty}\gamma^t R_t \right] = \Var^{\pistar_{\gamma}}\left[ \sum_{t=0}^{H-1} \gamma^t R_t + \gamma^H V_\gamma^{\pistar_{\gamma}}(S_H) \right] + \gamma^{2H} \left(P_{\pistar_{\gamma}}\right)^H \Var^{\pistar_{\gamma}} \left[ \sum_{t=0}^{\infty}\gamma^t R_t \right]
\]
(see our Lemma \ref{lem:multistep_variance_bellman_eqn}). This is a multi-step version of the standard variance Bellman equation (e.g. \cite[Theorem 1]{sobel_variance_1982}). Ordinarily an $H$-step expansion would not be useful, since the term $V_\gamma^{\pistar_{\gamma}}(S_H)$ by itself appears to generally have fluctuations on the order of $\frac{1}{1-\gamma}$ in the worst case depending on $S_H$ (note $S_H$ is the random state encountered at time $H$). However, in our setting, we should have $V_\gamma^{\pistar_{\gamma}}(S_H) \approx \frac{1}{1-\gamma}\rho^\star + h^\star(S_H)$, reducing the magnitude of the random fluctuations to order $H = \spannorm{h^\star}$. (See Lemma \ref{lem:discounted_value_span_bound} for a formalization of this approximation which first appeared in \cite{wei_model-free_2020}.) Therefore expansion to $H$ steps achieves the optimal tradeoff between maintaining $\Var^{\pistar_{\gamma}}\left[ \sum_{t=0}^{H-1} \gamma^t R_t + \gamma^H V_\gamma^{\pistar_{\gamma}}(S_H) \right] \leq O\left(H^2\right)$ and minimizing $\gamma^{2H}$. This yields the desired bound of
\[
\infnorm{\Var^{\pistar_{\gamma}}\left[ \sum_{t=0}^{\infty}\gamma^t R_t \right]} \leq O\left( \frac{H^2}{1-\gamma^{2H}}\right) = O\left( \frac{H}{1-\gamma}\right),
\]
where the fact that $\frac{1}{1-\gamma^{2H}} \leq O\left(\frac{1}{H(1-\gamma)}\right)$ is an elementary fact that requires our assumption of a large effective horizon (a $\gamma$ large enough that $\frac{1}{1-\gamma} \geq H$). See Lemma \ref{lem:pistar_var_bound} for the precise bound on this variance parameter.

We would like to use a similar argument as above to bound the other variance parameter, namely $\infnorm{(I - \gamma P_{\pihstar_{\gamma, \pert}})^{-1} \sqrt{\Var_{P_{\pihstar_{\gamma, \pert}}} \left[V_{\gamma, \pert}^{\pihstar_{\gamma, \pert}} \right]}}$. The first issue is that we must analyze the variance in the MDP with perturbed reward vector $\rpert$, but for suitably small perturbation, this is a relatively minor problem. The more significant problem is relating the policy $\pihstar_{\gamma, \pert}$ back to $\pistar_\gamma$. Since $H \leq O\left(\tmix \right)$, the arguments from the previous paragraph immediately imply a bound involving $\tmix$ in place of $H$, and if we assumed a uniform mixing time for all policies, then such arguments would easily apply to $\pihstar_{\gamma, \pert}$ in place of $\pistar_\gamma$. However, we hope to characterize the sample complexity in terms of our bound $H$ on $\spannorm{h^\star}$, and there is no a priori relationship between the variance of $V^{\pihstar_{\gamma, \pert}}_\gamma(S_H)$ and $\spannorm{h^\star} = \spannorm{h^{\pistar}}$. To overcome this issue, we bound variance parameters involving $\pihstar_{\gamma, \pert}$ in terms of both $H$ and the suboptimality $\infnorm{V^{\pistar_\gamma}_\gamma - V^{\pihstar_{\gamma, \pert}}_\gamma}$. This leads to a recursive bound on the suboptimality of $\pihstar_{\gamma, \pert}$ which ultimately obtains the desired sample complexity.

\section{Proofs}

In this section, we provide the proofs for our main results in Section~\ref{sec:results}.
Before beginning, we note that given that $H \geq 1$, we may assume that $H$ is an integer by setting $H \gets \lceil H \rceil$, which only affects the sample complexity by a constant multiple $<2$ relative to the original parameter $H$. Let $\infinfnorm{M} := \sup_{v: \infnorm{v} \leq 1} \infnorm{Mv}$ denote the $\ell_\infty$ operator norm of a matrix $M$. We record the standard and useful fact that $\infinfnorm{(I - \gamma P')^{-1}} \leq \frac{1}{1-\gamma}$ for any transition probability matrix $P'$, which follows from the Neumann series $(I - \gamma P')^{-1} = \sum_{t \geq 0} \left(\gamma P' \right)^t$ and then the triangle inequality for $\infinfnorm{\cdot}$ and the elementary fact that $\infinfnorm{P'}\le 1$.

\subsection{Technical Lemmas}

First we formally state the main theorem from \cite{wang_near_2022}, which gives a reduction from average-reward problems to discounted problems.
\begin{lem}
\label{lem:DMDP_reduction}
    Suppose $(P, r)$ is an MDP which is weakly communicating and has an optimal bias function $h^{\star}$ satisfying $\spannorm{h^{\star}} \leq H$. Fix $\varepsilon \in (0, 1]$ and set $\gamma = 1 - \frac{\varepsilon}{H}$. For any $\varepsilon_{\gamma} \in [0, \frac{1}{1-\gamma}]$, if $\pi$ is any $\varepsilon_{\gamma}$-optimal policy for the discounted MDP $(P, r, \gamma)$, then
    \begin{align*}
        \rho^{\star} - \rho^\pi \leq \left(8 + 3\frac{\varepsilon_{\gamma}}{H} \right)\varepsilon \one.
    \end{align*}
\end{lem}

From here, we will first establish lemmas which are useful for proving Theorem \ref{thm:DMDP_bound}, and then we will apply the reduction approach of Lemma \ref{lem:DMDP_reduction} to prove Theorem \ref{thm:main_theorem}. As mentioned in the introduction, a key technical component of our approach is to establish superior bounds on a certain instance-dependent variance quantity which replace a factor of $\frac{1}{1-\gamma}$ with a factor of $H$. Before reaching this step however, to make use of such a bound, we require an algorithm for discounted MDPs which enjoys a variance-dependent guarantee.

The work \cite{li_breaking_2020} obtains bounds with variance dependence that suffice for our purposes. However, they do not directly present said variance-dependent bounds, so we must slightly repackage their arguments in the form we require.
\begin{lem}
\label{lem:DMDP_error_bounds}

    There exist absolute constants $c_1, c_2$ such that for any $\delta \in (0,1)$, if $n \geq \frac{c_2}{1-\gamma}\log \left(\frac{S A}{(1-\gamma)\delta \varepsilon}\right) $, then with probability at least $1-\delta$, after running Algorithm \ref{alg:DMDP_alg}, we have
    \begin{align}
        \infnorm{\Vhat_{\gamma, \pert}^{\pistar_\gamma} - V_\gamma^{\pistar_\gamma}} & \leq  \gamma \sqrt{\frac{c_1 \log \left( \frac{S A}{(1-\gamma)\delta \varepsilon}\right)}{n}} \infnorm{(I - \gamma P_{\pistar_\gamma})^{-1} \sqrt{\Var_{P_{\pistar_\gamma}} \left[V_\gamma^{\pistar_\gamma} \right]}} + c_1 \gamma \frac{\log \left( \frac{S A}{(1-\gamma)\delta \varepsilon}\right)}{(1-\gamma)n} \infnorm{V_\gamma^{\pistar_\gamma}} + \frac{\varepsilon}{6} \label{eq:pistar_error}
    \end{align}
    and
    \begin{align}
        \infnorm{\Vhat_{\gamma, \pert}^{\pihstar_{\gamma, \pert}} - V_\gamma^{\pihstar_{\gamma, \pert}}} & \leq  \gamma \sqrt{\frac{c_1\log \left( \frac{S A}{(1-\gamma)\delta \varepsilon}\right)}{n}} \infnorm{(I - \gamma P_{\pihstar_{\gamma, \pert}})^{-1} \sqrt{\Var_{P_{\pihstar_{\gamma, \pert}}} \left[V_{\gamma, \pert}^{\pihstar_{\gamma, \pert}} \right]}} + c_1 \gamma\frac{\log \left( \frac{S A}{(1-\gamma)\delta \varepsilon}\right)}{(1-\gamma)n} \infnorm{V_{\gamma, \pert}^{\pihstar_{\gamma, \pert}}} + \frac{\varepsilon}{6}. \label{eq:pihatstar_error}
    \end{align}
\end{lem}
\begin{proof}
    First we establish equation~\eqref{eq:pistar_error}. The proof of \cite[Lemma 1]{li_breaking_2020} shows that when $n \geq \frac{16 e^2}{1-\gamma}2\log \left( \frac{4S \log \frac{e}{1-\gamma}}{\delta}\right)$, with probability at least $1-\delta$ we have
    \begin{align}
        \infnorm{\Vhat_{\gamma}^{\pistar_\gamma} - V_\gamma^{\pistar_\gamma}} & \leq 4 \gamma \sqrt{\frac{2\log \left( \frac{4S \log \frac{e}{1-\gamma}}{\delta }\right)}{n}} \infnorm{(I - \gamma P_{\pistar_\gamma})^{-1} \sqrt{\Var_{P_{\pistar_\gamma}} \left[V_\gamma^{\pistar_\gamma} \right]}} +  \gamma \frac{2\log \left( \frac{4S \log \frac{e}{1-\gamma}}{\delta }\right)}{(1-\gamma)n} \infnorm{V_\gamma^{\pistar_\gamma}}. \label{eq:almost_pistar_error}
    \end{align}
    Now since
    \begin{align*}
        \infnorm{\Vhat_{\gamma, \pert}^{\pistar_\gamma} - \Vhat_{\gamma}^{\pistar_\gamma}} &= \infnorm{(I- \gamma \Phat_{\pistar_\gamma})^{-1}{\rpert}_{\pistar_\gamma} - (I- \gamma \Phat_{\pistar_\gamma})^{-1}{r}_{\pistar_\gamma}}\\
        &\leq \infinfnorm{(I- \gamma \Phat_{\pistar_\gamma})^{-1}} \infnorm{\rpert - r}\\
        &\leq \frac{\xi}{1-\gamma} = \frac{\varepsilon}{6},
    \end{align*}
    we can obtain equation~\eqref{eq:pistar_error} by triangle inequality (although we will choose the constant $c_1$ below).

    Next we establish equation~\eqref{eq:pihatstar_error}. Using \cite[Lemma 6]{li_breaking_2020}, with probability at least $1-\delta$ we have that
    \begin{align}
        \left| \Qhat^\star_{\gamma, \pert} (s, \pihstar_{\gamma, \pert}(s)) - \Qhat^\star_{\gamma, \pert} (s, a)\right| > \frac{\xi \delta (1-\gamma)}{3 S A^2} = \frac{\varepsilon \delta (1-\gamma)^2}{18 S A^2} \label{eq:separation_cond}
    \end{align}
    uniformly over all $s$ and all $a \neq \pihstar_{\gamma, \pert}(s)$. From this separation condition~\eqref{eq:separation_cond}, the assumptions of \cite[Lemma 5]{li_breaking_2020} hold (with $\omega = \frac{\varepsilon \delta (1-\gamma)^2}{18 S A^2}$ in their notation) for the MDP with the perturbed reward $\rpert$. The proof of \cite[Lemma 5]{li_breaking_2020} shows that under the event~\eqref{eq:separation_cond} holds, the conditions for \cite[Lemma 2]{li_breaking_2020} are satisfied (with, in their notation, $\beta_1 = 2 \log \left( \frac{32}{(1-\gamma)^2 \omega \delta} SA \log \frac{e}{1-\gamma}\right) = 2 \log \left( \frac{576 S^2A^3}{(1-\gamma)^4 \delta^2 \varepsilon} \log \frac{e}{1-\gamma}\right)$) with additional failure probability $\leq \delta$. The proof of \cite[Lemma 2]{li_breaking_2020} then shows that, assuming $n > \frac{16 e^2}{1-\gamma}2 \log \left( \frac{576 S^2A^3}{(1-\gamma)^4 \delta^2 \varepsilon}  \log \frac{e}{1-\gamma}\right)$, we have
    \begin{align}
        \infnorm{\Vhat_{\gamma, \pert}^{\pihstar_{\gamma, \pert}} - V_{\gamma, \pert}^{\pihstar_{\gamma, \pert}}} & \leq 4\gamma\sqrt{\frac{\beta_1}{n}} \infnorm{(I - \gamma P_{\pihstar_{\gamma, \pert}})^{-1}\sqrt{\Var_{P_{\pihstar_{\gamma, \pert}}}}\left[ V_{\gamma, \pert}^{\pihstar_{\gamma, \pert}} \right]  } + \frac{\gamma \beta_1}{(1-\gamma)n} \infnorm{V_{\gamma, \pert}^{\pihstar_{\gamma, \pert}}} \label{eq:almost_pihatstar_error}
    \end{align}
    where we abbreviated $\beta_1 = 2 \log \left( \frac{576 S^2A^3}{(1-\gamma)^4 \delta^2 \varepsilon} \log \frac{e}{1-\gamma}\right)$ for notational convenience.

    We can again calculate that
    \begin{align*}
        \infnorm{V_{\gamma, \pert}^{\pihstar_{\gamma, \pert}} - V_{\gamma}^{\pihstar_{\gamma, \pert}} } &= \infnorm{(I - \gamma P_{\pihstar_{\gamma, \pert}})^{-1} {\rpert}_{\pihstar_{\gamma, \pert}} - (I - \gamma P_{\pihstar_{\gamma, \pert}})^{-1} {r}_{\pihstar_{\gamma, \pert}}} \\
        & \leq \infinfnorm{(I - \gamma P_{\pihstar_{\gamma, \pert}})^{-1}} \infnorm{{\rpert} - {r}}\\
        & \leq \frac{\xi}{1-\gamma} = \frac{\varepsilon}{6}
    \end{align*}
    so $ \infnorm{\Vhat_{\gamma, \pert}^{\pihstar_{\gamma, \pert}} - V_\gamma^{\pihstar_{\gamma, \pert}}} \leq \infnorm{\Vhat_{\gamma, \pert}^{\pihstar_{\gamma, \pert}} - V_{\gamma, \pert}^{\pihstar_{\gamma, \pert}}} + \frac{\varepsilon}{6}$ by triangle inequality, essentially giving~\eqref{eq:pihatstar_error}.

    Finally, to choose the constants $c_1$ and $c_2$, we first note that $2\log \left( \frac{4S \log \frac{e}{1-\gamma}}{\delta }\right) \leq \beta_1 < c_1' \log \left(\frac{SA}{(1-\gamma) \delta \varepsilon} \right)$ for some absolute constant $c_1'$, and therefore also all our requirements on $n$ are fulfilled when $n \geq \frac{16e^2}{1-\gamma}c_1' \log \left(\frac{SA}{(1-\gamma) \delta \varepsilon} \right) = \frac{c_2'}{1-\gamma}\log \left(\frac{SA}{(1-\gamma) \delta \varepsilon} \right)$ for another absolute constant $c_2'$. Lastly we note that by the union bound the total failure probability is at most $3\delta$, so to obtain a failure probability of $\delta'$ we may set $\delta = \delta'/3$ and absorb the additional constant when defining $c_1, c_2$ in terms of $c_1', c_2'$, and we also then increase $c_1$ by a factor of $4$ to absorb the factor of $4$ appearing in the first terms within~\eqref{eq:almost_pistar_error} and~\eqref{eq:almost_pihatstar_error}.
\end{proof}

Now we can analyze the variance parameters 
\begin{equation*}
\infnorm{(I - \gamma P_{\pistar_\gamma})^{-1} \sqrt{\Var_{P_{\pistar_\gamma}} \left[V_\gamma^{\pistar_\gamma} \right]}}
\quad\text{ and }\quad
\infnorm{(I - \gamma P_{\pihstar_{\gamma, \pert}})^{-1} \sqrt{\Var_{P_{\pihstar_{\gamma, \pert}}} \left[V_{\gamma, \pert}^{\pihstar_{\gamma, \pert}} \right]}},
\end{equation*}
which appear in Lemma \ref{lem:DMDP_error_bounds}.
We reproduce the following lemma from \cite[Lemma 2]{wei_model-free_2020}.
\begin{lem}
\label{lem:discounted_value_span_bound}
It holds that
    \[\sup_s \left|V_\gamma^{\pistar_\gamma}(s) - \frac{1}{1-\gamma}\rho^\star \right| \leq H.\]
\end{lem}

The following gives a relationship between different variance parameters.
This result essentially appears in \cite[Lemma 4]{agarwal_model-based_2020} (which was in turn inspired by \cite[Lemma 8]{gheshlaghi_azar_minimax_2013}), but since their result pertains to objects slightly different than $P_\pi$ and $\Var_{P_{\pi}} \left[V_\gamma^{\pi} \right]$, we provide the full argument for completeness.
\begin{lem}
\label{lem:var_params_relationship}
    For any deterministic stationary policy $\pi$, we have
    \begin{align*}
         \gamma \infnorm{(I - \gamma P_{\pi})^{-1} \sqrt{\Var_{P_{\pi}} \left[V_\gamma^{\pi} \right] }} & \leq \sqrt{\frac{2}{1-\gamma}} \sqrt{\infnorm{\Var^\pi\left[ \sum_{t=0}^{\infty} \gamma^t R_t  \right]}}
    \end{align*}
\end{lem}
\begin{proof}
    First we note the well-known variance Bellman equation (see for instance \cite[Theorem 1]{sobel_variance_1982})
    \begin{align}
        \Var^\pi\left[ \sum_{t=0}^{\infty} \gamma^t R_t  \right] = \gamma^2\Var_{P_{\pi}} \left[V_\gamma^{\pi} \right] + \gamma^2 P_\pi \Var^\pi\left[ \sum_{t=0}^{\infty} \gamma^t R_t  \right]. \label{eq:variance_bellman_eqn}
    \end{align}
    Now we can basically identically follow the argument of \cite[Lemma 4]{agarwal_model-based_2020}. $(1-\gamma)(I - \gamma P_{\pi})^{-1}$ has rows which are each probability distributions (are non-negative and sum to $1$), so by Jensen's inequality, for each row $s\in \S$, since $\sqrt{\cdot}$ is concave,
    \[
    \left|(1-\gamma)e_s^\top(I - \gamma P_{\pi})^{-1} \sqrt{\Var_{P_{\pi}} \left[V_\gamma^{\pi} \right] } \right|\leq \sqrt{ \left|(1-\gamma)e_s^\top(I - \gamma P_{\pi})^{-1} \Var_{P_{\pi}} \left[V_\gamma^{\pi} \right]  \right|}.
    \]
    Using this fact we can calculate that, abbreviating $v = \Var_{P_{\pi}} \left[V_\gamma^{\pi} \right]$,
    \begin{align*}
        \gamma \infnorm{(I - \gamma P_{\pi})^{-1} \sqrt{v }} 
        &= \gamma \frac{1}{1-\gamma} \infnorm{(1-\gamma)(I - \gamma P_{\pi})^{-1} \sqrt{v }} \\
        &\leq \gamma \frac{1}{1-\gamma} \sqrt{\infnorm{(1-\gamma)(I - \gamma P_{\pi})^{-1} v }} \\
        &= \gamma \frac{1}{\sqrt{1-\gamma}} \sqrt{\infnorm{(I - \gamma P_{\pi})^{-1} v }}.
    \end{align*}

    In order to relate $\infnorm{(I - \gamma P_{\pi})^{-1} v }$ to $\infnorm{(I - \gamma^2 P_{\pi})^{-1} v }$ in order to apply the variance Bellman equation~\eqref{eq:variance_bellman_eqn}, we calculate
    \begin{align*}
        \infnorm{(I - \gamma P_{\pi})^{-1} v } &= \infnorm{(I - \gamma P_{\pi})^{-1} (I - \gamma^2 P_{\pi}) (I - \gamma^2 P_{\pi})^{-1}v } \\
        &=  \infnorm{(I - \gamma P_{\pi})^{-1} \left((1-\gamma)I + \gamma(I - \gamma P_\pi) \right) (I - \gamma^2 P_{\pi})^{-1}v } \\
        &=  \infnorm{\left((1-\gamma) (I - \gamma P_\pi)^{-1} + \gamma I \right) (I - \gamma^2 P_{\pi})^{-1}v } \\
        &\leq  \infnorm{(1-\gamma) (I - \gamma P_\pi)^{-1} (I - \gamma^2 P_{\pi})^{-1}v } + \gamma \infnorm{(I - \gamma^2 P_\pi)^{-1} v } \\
        &\leq (1-\gamma) \infinfnorm{(I - \gamma P_\pi)^{-1}} \infnorm{(I - \gamma^2 P_{\pi})^{-1}v } + \gamma \infnorm{(I - \gamma^2 P_\pi)^{-1} v } \\
        &\leq  (1+\gamma) \infnorm{(I - \gamma^2 P_{\pi})^{-1}v }  \\
        &\leq  2 \infnorm{(I - \gamma^2 P_{\pi})^{-1}v }
    \end{align*}
    Combining these calculations with the variance Bellman equation~\eqref{eq:variance_bellman_eqn}, we conclude that
    \begin{align*}
        \gamma \infnorm{(I - \gamma P_{\pi})^{-1} \sqrt{v }} \leq \gamma \frac{1}{\sqrt{1-\gamma}} \sqrt{2 \infnorm{(I - \gamma^2 P_{\pi})^{-1}v }} \leq \sqrt{\frac{2}{1-\gamma}} \sqrt{\infnorm{\Var^\pi\left[ \sum_{t=0}^{\infty} \gamma^t R_t  \right]}}
    \end{align*}
    as desired.
\end{proof}

The following is a multi-step version of the variance Bellman equation, which we will later apply with $T = H$ but holds for arbitrary $T$.
\begin{lem}
\label{lem:multistep_variance_bellman_eqn}
    For any integer $T \geq 1$, for any deterministic stationary policy $\pi$, we have
    \begin{align*}
        \infnorm{\Var^\pi\left[ \sum_{t=0}^{\infty} \gamma^t R_t  \right]} & \leq \frac{\infnorm{\Var^\pi\left[ \sum_{t=0}^{T-1} \gamma^t R_t + \gamma^T V_\gamma^\pi(S_T) \right]} }{1 - \gamma^{2T}}
    \end{align*}
\end{lem}
\begin{proof}
    Fix a state $s_0 \in \S$. Letting $\mathcal{F}_T$ be the $\sigma$-algebra generated by $(S_1, \dots, S_T)$, we calculate that
    \begin{align*}
    \Var^\pi_{s_0}\left[ \sum_{t=0}^{\infty} \gamma^t R_t  \right] 
    &= \E^\pi_{s_0} \left( \sum_{t=0}^{\infty} \gamma^t R_t - V_\gamma^\pi(s_0) \right)^2 \\
    &= \E^\pi_{s_0} \left( \sum_{t=0}^{T-1} \gamma^t R_t + \gamma^T V_\gamma^\pi(S_T)   - V_\gamma^\pi(s_0) + \sum_{t=T}^{\infty} \gamma^t R_t - \gamma^T V_\gamma^\pi(S_T) \right)^2 \\
    &= \E^\pi_{s_0} \Bigg[\E^\pi_{s_0} \Bigg[  \Bigg( \underbrace{\sum_{t=0}^{T-1} \gamma^t R_t + \gamma^T V_\gamma^\pi(S_T)   - V_\gamma^\pi(s_0)}_\text{$A$} + \underbrace{\sum_{t=T}^{\infty} \gamma^t R_t - \gamma^T V_\gamma^\pi(S_T)}_\text{$B$} \Bigg)^2 \Bigg| \mathcal{F}_{T}\Bigg] \Bigg]\\
    &= \E^\pi_{s_0} \left[\E^\pi_{s_0} \left[  A^2 + B^2 + 2AB \middle| \mathcal{F}_{T}\right] \right]\\
    &= \E^\pi_{s_0} \left[A^2 + \E^\pi_{s_0} \left[ B^2  \middle| \mathcal{F}_{T}\right] + 2A\E^\pi_{s_0} \left[ B \middle| \mathcal{F}_{T}\right] \right] \\
    &= \E^\pi_{s_0} \left[A^2 + \E^\pi_{S_T} \left[ B^2 \right] \right] \\
    &= \E^\pi_{s_0} \left[\left( \sum_{t=0}^{T-1} \gamma^t R_t + \gamma^T V_\gamma^\pi(S_T)   - V_\gamma^\pi(s_0) \right)^2 + \E^\pi_{S_T} \left[ \left( \sum_{t=T}^{\infty} \gamma^t R_t - \gamma^T V_\gamma^\pi(S_T) \right)^2 \right] \right] \\
    &= \E^\pi_{s_0} \left[\left( \sum_{t=0}^{T-1} \gamma^t R_t + \gamma^T V_\gamma^\pi(S_T)   - V_\gamma^\pi(s_0) \right)^2 + \gamma^{2T} \E^\pi_{S_T} \left[ \left( \sum_{t=0}^{\infty} \gamma^t R_t - V_\gamma^\pi(S_T) \right)^2 \right] \right] \\
    &= \Var^\pi_{s_0}\left[ \sum_{t=0}^{T-1} \gamma^t R_t + \gamma^T V_\gamma^\pi(S_T) \right] + \gamma^{2T} e_{s_0}^\top P_\pi^T  \Var^\pi\left[ \sum_{t=0}^{\infty} \gamma^t R_t  \right]
\end{align*}
where we used the tower property, the Markov property, and the fact that $\E^\pi_{s_0} \left[ B \middle| \mathcal{F}_{T}\right] = 0$ (which is immediate from the definition of $V_\gamma^\pi$).
Since $e_{s_0}^\top P_\pi^T$ is a probability distribution, by Holder's inequality $\left|e_{s_0}^\top P_\pi^T  \Var^\pi\left[ \sum_{t=0}^{\infty} \gamma^t R_t  \right]\right| \leq \infnorm{\Var^\pi\left[ \sum_{t=0}^{\infty} \gamma^t R_t  \right]}$. Therefore
\begin{align*}
    \infnorm{\Var^\pi_{s_0}\left[ \sum_{t=0}^{\infty} \gamma^t R_t  \right]} & \leq \infnorm{\Var^\pi\left[ \sum_{t=0}^{T-1} \gamma^t R_t + \gamma^T V_\gamma^\pi(S_T) \right]} + \gamma^{2T}\infnorm{\Var^\pi_{s_0}\left[ \sum_{t=0}^{\infty} \gamma^t R_t  \right]}
\end{align*}
and we can obtain the desired conclusion after rearranging terms.
\end{proof}

\begin{lem}
    \label{lem:long_horizon_gamma_ineq}
    If $\gamma \geq 1 - \frac{1}{H}$ for some integer $H \geq 1$, then
    \begin{align*}
        \frac{1-\gamma^{2H}}{1-\gamma} \geq \left(1-\frac{1}{e^2} \right)H \geq \frac{4}{5}H.
    \end{align*}
\end{lem}
\begin{proof}
    Fixing $H \geq 1$, we have
    \[\frac{1-\gamma^{2H}}{1-\gamma} = 1 + \gamma + \gamma^2 + \dots + \gamma^{2H-1} \]
    which is increasing in $\gamma$, so $\inf_{\gamma \geq 1 - \frac{1}{H}} \frac{1-\gamma^{2H}}{1-\gamma}$ is attained at $\gamma = 1 - \frac{1}{H}$.
    Now allowing $H \geq 1$ to be arbitrary, note $\frac{1-\left(1-\frac{1}{H} \right)^{2H}}{1-\left(1-\frac{1}{H} \right)} = H\left(1-\left(1-\frac{1}{H} \right)^{2H}\right)$ so it suffices to show that $1-\left(1-\frac{1}{H} \right)^{2H} \geq 1-e^2$ for all $H \geq 1$. By computing the derivative, one finds that $1-\left(1-\frac{1}{H} \right)^{2H}$ is monotonically decreasing, so
    \[1-\left(1-\frac{1}{H} \right)^{2H} \geq \lim_{H \to \infty} 1-\left(1-\frac{1}{H} \right)^{2H} = 1-\frac{1}{e^2}.\]
\end{proof}

\begin{lem}
    \label{lem:pistar_var_bound}
    Letting $\pistar_{\gamma}$ be the optimal policy for the discounted MDP $(P, r,\gamma)$, we have
    \begin{align*}
        \infnorm{\Var^{\pistar_{\gamma}}\left[ \sum_{t=0}^{\infty}\gamma^t R_t \right]} &\leq 5 \frac{H}{1-\gamma}
    \end{align*}
\end{lem}
\begin{proof}
    By using Lemma \ref{lem:multistep_variance_bellman_eqn}, it suffices to bound $\infnorm{\Var^{\pistar_{\gamma}}\left[ \sum_{t=0}^{H-1} \gamma^t R_t + \gamma^H V_\gamma^{\pistar_{\gamma}}(S_H) \right]}$.

    Fixing a state $s_0 \in \S$,
    \begin{align*}
        \Var^{\pistar_{\gamma}}_{s_0}\left[ \sum_{t=0}^{H-1} \gamma^t R_t + \gamma^H V_\gamma^{\pistar_{\gamma}}(S_H) \right] &= \Var^{\pistar_{\gamma}}_{s_0}\left[ \sum_{t=0}^{H-1} \gamma^t R_t + \gamma^H \left(V_\gamma^{\pistar_{\gamma}}(S_H) - \frac{1}{1-\gamma}\rho^\star \right) \right] \\
        &\leq \E^{\pistar_{\gamma}}_{s_0}\left| \sum_{t=0}^{H-1} \gamma^t R_t + \gamma^H \left(V_\gamma^{\pistar_{\gamma}}(S_H) - \frac{1}{1-\gamma}\rho^\star \right) \right|^2 \\
        & \leq 2\E^{\pistar_{\gamma}}_{s_0}\left| \sum_{t=0}^{H-1} \gamma^t R_t \right|^2+ 2\E^{\pistar_{\gamma}}_{s_0}\left| \gamma^H \left(V_\gamma^{\pistar_{\gamma}}(S_H) - \frac{1}{1-\gamma}\rho^\star \right) \right|^2 \\
        & \leq 2H^2 + 2 \sup_s \left(V_{\pistar_{\gamma}}^\pi(s) - \frac{1}{1-\gamma}\rho^\star \right)^2 \\
        & \leq 4H^2
    \end{align*}
    where in the final inequality we used Lemma \ref{lem:discounted_value_span_bound}. Taking the maximum over all states $s$ and combining with Lemma \ref{lem:multistep_variance_bellman_eqn} we obtain
    \begin{align*}
        \infnorm{\Var^{\pistar_{\gamma}}\left[ \sum_{t=0}^{\infty}\gamma^t R_t \right]} &\leq \frac{4H^2}{1 - \gamma^{2H}}.
    \end{align*}

    Now combining this with Lemma \ref{lem:long_horizon_gamma_ineq}, which can be rearranged to show that $\frac{1}{1-\gamma^{2H}} \leq \frac{5}{4}\frac{H}{1-\gamma}$, we complete the proof. 
\end{proof}

\begin{lem}
    \label{lem:pihatstar_var_bound}
    We have
    \begin{align*}
        \infnorm{\Var^{\pihstar_{\gamma, \pert}}\left[ \sum_{t=0}^{\infty} \gamma^t {\Rpert}_t  \right]}
    & \leq 15\frac{H^2 + \infnorm{V_\gamma^{\pihstar_{\gamma, \pert}} - \Vhat_{\gamma, \pert}^{\pihstar_{\gamma, \pert}}}^2 + \infnorm{V_\gamma^{\pistar_\gamma} - \Vhat_{\gamma, \pert}^{\pistar_\gamma}}^2}{H(1-\gamma)} .
    \end{align*}
\end{lem}
\begin{proof}
    In light of Lemma \ref{lem:multistep_variance_bellman_eqn}, it suffices to give a bound on $\infnorm{\Var^{\pihstar_{\gamma, \pert}}\left[ \sum_{t=0}^{H-1} \gamma^t {\Rpert}_t + \gamma^H V_{\gamma, \pert}^{\pihstar_{\gamma, \pert}}(S_H) \right]}$. We have for any state $s_0$ that
\begin{align}
    &\quad \Var_{s_0}^{\pihstar_{\gamma, \pert}}\left[ \sum_{t=0}^{H-1} \gamma^t {\Rpert}_t + \gamma^H V_{\gamma, \pert}^{\pihstar_{\gamma, \pert}}(S_H) \right] \nonumber \\
    &= \Var_{s_0}^{\pihstar_{\gamma, \pert}}\left[ \sum_{t=0}^{H-1} \gamma^t {\Rpert}_t + \gamma^H V_{\gamma, \pert}^{\pihstar_{\gamma, \pert}}(S_H) -\gamma^H \frac{1}{1-\gamma}\rho^\star \right] \nonumber \\
    &\leq \E_{s_0}^{\pihstar_{\gamma, \pert}}\left( \sum_{t=0}^{H-1} \gamma^t {\Rpert}_t + \gamma^H V_{\gamma, \pert}^{\pihstar_{\gamma, \pert}}(S_H) -\gamma^H \frac{1}{1-\gamma}\rho^\star \right)^2 \nonumber \\
    &= \E_{s_0}^{\pihstar_{\gamma, \pert}}\left( \sum_{t=0}^{H-1} \gamma^t {\Rpert}_t + \gamma^H \left( V_{\gamma, \pert}^{\pihstar_{\gamma, \pert}}(S_H) - V_\gamma^{\pistar_\gamma}(S_H)  \right) + \gamma^H \left(V_\gamma^{\pistar_\gamma}(S_H) - \frac{1}{1-\gamma}\rho^\star \right)\right)^2 \nonumber \\
    &\leq 3\E_{s_0}^{\pihstar_{\gamma, \pert}}\left( \sum_{t=0}^{H-1} \gamma^t {\Rpert}_t \right)^2 + 3\gamma^{2H}\E_{s_0}^{\pihstar_{\gamma, \pert}}\left(   V_{\gamma, \pert}^{\pihstar_{\gamma, \pert}}(S_H) - V_\gamma^{\pistar_\gamma}(S_H) \right)^2  \nonumber \\
    & \qquad + 3\gamma^{2H}\E_{s_0}^{\pihstar_{\gamma, \pert}} \left(V_\gamma^{\pistar_\gamma}(S_H) - \frac{1}{1-\gamma}\rho^\star \right)^2 \nonumber \\
    &\leq 3\E_{s_0}^{\pihstar_{\gamma, \pert}}\left( \sum_{t=0}^{H-1} \gamma^t {\Rpert}_t \right)^2 + 6\gamma^{2H}\E_{s_0}^{\pihstar_{\gamma, \pert}}\left(   V_{\gamma}^{\pihstar_{\gamma, \pert}}(S_H) - V_\gamma^{\pistar_\gamma}(S_H) \right)^2 + 6\gamma^{2H} \infnorm{V_{\gamma, \pert}^{\pihstar_{\gamma, \pert}} - V_{\gamma}^{\pihstar_{\gamma, \pert}}}^2\nonumber \\
    & \qquad + 3\gamma^{2H}\E_{s_0}^{\pihstar_{\gamma, \pert}} \left(V_\gamma^{\pistar_\gamma}(S_H) - \frac{1}{1-\gamma}\rho^\star \right)^2 ,\label{eq:pihatstar_var_bound_step1}
\end{align}
where we have used triangle inequality and the inequalities $(a+b)^2 \leq 2a^2 + 2b^2$ and $(a+b+c)^2 \leq 3a^2 + 3b^2 + 3c^2$. Now we bound each term of~\eqref{eq:pihatstar_var_bound_step1}. First,
\begin{align*}
    3\E_{s_0}^{\pihstar_{\gamma, \pert}}\left( \sum_{t=0}^{H-1} \gamma^t {\Rpert}_t \right)^2 & \leq 3 \left(H \infnorm{\rpert} \right)^2 \leq 3H^2 (\infnorm{r} + \xi)^2 \leq 6 H^2 \left( 1 + \left(\frac{(1-\gamma)\varepsilon}{6}\right)^2 \right) \leq 6 H^2 \left(\frac{7}{6}\right)^2 ,
\end{align*}
where we had $\frac{(1-\gamma)\varepsilon}{6} \leq \frac{\varepsilon}{6H} \leq \frac{1}{6}$ because $\frac{1}{1-\gamma} \geq H$ and $\varepsilon \leq H$. Clearly $6\gamma^{2H}\E_{s_0}^{\pihstar_{\gamma, \pert}}\left(   V_{\gamma}^{\pihstar_{\gamma, \pert}}(S_H) - V_\gamma^{\pistar_\gamma}(S_H) \right)^2 \leq 6\infnorm{V_{\gamma}^{\pihstar_{\gamma, \pert}} - V_\gamma^{\pistar_\gamma}}^2$. By an argument identical to those used in the proof of Lemma \ref{lem:DMDP_error_bounds},
\begin{align*}
    \infnorm{V_{\gamma, \pert}^{\pihstar_{\gamma, \pert}} - V_{\gamma}^{\pihstar_{\gamma, \pert}}} \leq \frac{1}{1-\gamma} \xi = \frac{\varepsilon}{6},
\end{align*}
so $6\gamma^{2H} \infnorm{V_{\gamma, \pert}^{\pihstar_{\gamma, \pert}} - V_{\gamma}^{\pihstar_{\gamma, \pert}}}^2 \leq \frac{\varepsilon^2}{6} \leq \frac{H^2}{6}$ since $\varepsilon \leq H$. Finally, using Lemma \ref{lem:discounted_value_span_bound},
\begin{align*}
    3\gamma^{2H}\E_{s_0}^{\pihstar_{\gamma, \pert}} \left(V_\gamma^{\pistar_\gamma}(S_H) - \frac{1}{1-\gamma}\rho^\star \right)^2 & \leq 3 \sup_s \left|V_\gamma^{\pistar_\gamma}(S_H) - \frac{1}{1-\gamma}\rho^\star \right|^2 \leq 3H^2.
\end{align*}
Using all these bounds in~\eqref{eq:pihatstar_var_bound_step1}, we have
\begin{align}
    &\quad \Var_{s_0}^{\pihstar_{\gamma, \pert}}\left[ \sum_{t=0}^{H-1} \gamma^t {\Rpert}_t + \gamma^H V_{\gamma, \pert}^{\pihstar_{\gamma, \pert}}(S_H) \right] \nonumber \\
    &\leq 3\E_{s_0}^{\pihstar_{\gamma, \pert}}\left( \sum_{t=0}^{H-1} \gamma^t {\Rpert}_t \right)^2 + 6\gamma^{2H}\E_{s_0}^{\pihstar_{\gamma, \pert}}\left(   V_{\gamma}^{\pihstar_{\gamma, \pert}}(S_H) - V_\gamma^{\pistar_\gamma}(S_H) \right)^2 + 6\gamma^{2H} \infnorm{V_{\gamma, \pert}^{\pihstar_{\gamma, \pert}} - V_{\gamma}^{\pihstar_{\gamma, \pert}}}^2 \nonumber\\
    & \qquad + 3\gamma^{2H}\E_{s_0}^{\pihstar_{\gamma, \pert}} \left(V_\gamma^{\pistar_\gamma}(S_H) - \frac{1}{1-\gamma}\rho^\star \right)^2 \nonumber\\
    & \leq \left(\frac{49}{6} + \frac{1}{6} + 3 \right)H^2 + 6\infnorm{V_{\gamma}^{\pihstar_{\gamma, \pert}} - V_\gamma^{\pistar_\gamma}}^2 \nonumber\\
    & \leq 12H^2 + 6\infnorm{V_{\gamma}^{\pihstar_{\gamma, \pert}} - V_\gamma^{\pistar_\gamma}}^2. \label{eq:pihatstar_var_bound_step2}
\end{align}

Finally, using the elementwise inequality
\begin{align*}
    V_\gamma^{\pistar_\gamma} &\geq V_\gamma^{\pihstar_{\gamma, \pert}} \\
    &\geq \Vhat_{\gamma, \pert}^{\pihstar_{\gamma, \pert}} - \infnorm{\Vhat_{\gamma, \pert}^{\pihstar_{\gamma, \pert}} - V_\gamma^{\pihstar_{\gamma, \pert}}}\one \\
    &\geq \Vhat_{\gamma, \pert}^{\pistar_{\gamma}} - \infnorm{\Vhat_{\gamma, \pert}^{\pihstar_{\gamma, \pert}} - V_\gamma^{\pihstar_{\gamma, \pert}}}\one \\
    & \geq V_{\gamma}^{\pistar_{\gamma}} - \infnorm{\Vhat_{\gamma, \pert}^{\pihstar_{\gamma, \pert}} - V_\gamma^{\pihstar_{\gamma, \pert}}}\one - \infnorm{\Vhat_{\gamma, \pert}^{\pistar_{\gamma}} - V_{\gamma}^{\pistar_{\gamma}}}\one,
\end{align*}
we see $\infnorm{V_{\gamma}^{\pihstar_{\gamma, \pert}} - V_\gamma^{\pistar_\gamma}} \leq \infnorm{\Vhat_{\gamma, \pert}^{\pihstar_{\gamma, \pert}} - V_\gamma^{\pihstar_{\gamma, \pert}}} + \infnorm{\Vhat_{\gamma, \pert}^{\pistar_{\gamma}} - V_{\gamma}^{\pistar_{\gamma}}}$. Combining this with~\eqref{eq:pihatstar_var_bound_step2}, we conclude
\begin{align}
    \Var_{s_0}^{\pihstar_{\gamma, \pert}}\left[ \sum_{t=0}^{H-1} \gamma^t {\Rpert}_t + \gamma^H V_{\gamma, \pert}^{\pihstar_{\gamma, \pert}}(S_H) \right] & \leq 12H^2 + 12 \infnorm{\Vhat_{\gamma, \pert}^{\pihstar_{\gamma, \pert}} - V_\gamma^{\pihstar_{\gamma, \pert}}}^2 + 12\infnorm{\Vhat_{\gamma, \pert}^{\pistar_{\gamma}} - V_{\gamma}^{\pistar_{\gamma}}}^2.
\end{align}

Now combining with Lemma \ref{lem:multistep_variance_bellman_eqn} and then using Lemma \ref{lem:long_horizon_gamma_ineq}, we have
\begin{align*}
    \infnorm{\Var^{\pihstar_{\gamma, \pert}}\left[ \sum_{t=0}^{\infty} \gamma^t {\Rpert}_t  \right]} & \leq \frac{\infnorm{\Var^{\pihstar_{\gamma, \pert}}\left[ \sum_{t=0}^{H-1} \gamma^t {\Rpert}_t + \gamma^H V_\gamma^{\pihstar_{\gamma, \pert}}(S_H) \right]}}{1-\gamma^{2H}} \\
    & \leq 12\frac{H^2 + \infnorm{V_\gamma^{\pihstar_{\gamma, \pert}} - \Vhat_\gamma^{\pihstar_{\gamma, \pert}}}^2 + \infnorm{V_\gamma^{\pistar_\gamma} - \Vhat_\gamma^{\pistar_\gamma}}^2}{1-\gamma^{2H}} \\
    & \leq 12 \frac{5}{4}\frac{H^2 + \infnorm{V_\gamma^{\pihstar_{\gamma, \pert}} - \Vhat_\gamma^{\pihstar_{\gamma, \pert}}}^2 + \infnorm{V_\gamma^{\pistar_\gamma} - \Vhat_\gamma^{\pistar_\gamma}}^2}{H(1-\gamma)} \\
    & = 15\frac{H^2 + \infnorm{V_\gamma^{\pihstar_{\gamma, \pert}} - \Vhat_\gamma^{\pihstar_{\gamma, \pert}}}^2 + \infnorm{V_\gamma^{\pistar_\gamma} - \Vhat_\gamma^{\pistar_\gamma}}^2}{H(1-\gamma)}
\end{align*}
as desired.
\end{proof}

\subsection{Proofs of Theorem~\ref{thm:DMDP_bound} and~\ref{thm:main_theorem}}

With the above lemmas we can complete the proof of Theorem \ref{thm:DMDP_bound}.
\begin{proof}[Proof of Theorem \ref{thm:DMDP_bound}]
    Our approach will be to utilize our variance bounds within the error bounds from Lemma \ref{lem:DMDP_error_bounds}. We will find a value for $n$ which guarantees that $\infnorm{\Vhat_{\gamma, \pert}^{\pistar_\gamma} - V_\gamma^{\pistar_\gamma}}$ and $ \infnorm{\Vhat_{\gamma, \pert}^{\pihstar_{\gamma, \pert}} - V_\gamma^{\pihstar_{\gamma, \pert}}}$ are both $\leq \varepsilon/2$, which guarantees that $\infnorm{V_\gamma^{\pihstar_{\gamma, \pert}} - V_\gamma^{\pistar_\gamma}} \leq \varepsilon$. 

    First we note that the conclusions of Lemma \ref{lem:DMDP_error_bounds} require $n \geq \frac{c_2}{1-\gamma}\log \left(\frac{S A}{(1-\gamma)\delta \varepsilon}\right) $ so we assume $n$ is large enough that this holds.
    
    Now we bound $\infnorm{\Vhat_{\gamma, \pert}^{\pistar_\gamma} - V_\gamma^{\pistar_\gamma}}$. Starting with inequality~\eqref{eq:pistar_error} from Lemma \ref{lem:DMDP_error_bounds} and then applying our variance bounds through Lemma \ref{lem:var_params_relationship} and then Lemma \ref{lem:pistar_var_bound}, we have
    \begin{align*}
        \infnorm{\Vhat_{\gamma, \pert}^{\pistar_\gamma} - V_\gamma^{\pistar_\gamma}} & \leq  \gamma \sqrt{\frac{c_1\log \left( \frac{S A}{(1-\gamma)\delta \varepsilon}\right)}{n}} \infnorm{(I - \gamma P_{\pistar_\gamma})^{-1} \sqrt{\Var_{P_{\pistar_\gamma}} \left[V_\gamma^{\pistar_\gamma} \right]}} + c_1 \gamma \frac{\log \left( \frac{S A}{(1-\gamma)\delta \varepsilon}\right)}{(1-\gamma)n} \infnorm{V_\gamma^{\pistar_\gamma}} + \frac{\varepsilon}{6} \\
        & \leq   \sqrt{\frac{c_1\log \left( \frac{S A}{(1-\gamma)\delta \varepsilon}\right)}{n}} \sqrt{\frac{2}{1-\gamma}} \sqrt{\infnorm{\Var^{\pistar_{\gamma}}\left[ \sum_{t=0}^{\infty}\gamma^t R_t \right]}}  + c_1 \gamma \frac{\log \left( \frac{S A}{(1-\gamma)\delta \varepsilon}\right)}{(1-\gamma)n} \infnorm{V_\gamma^{\pistar_\gamma}} + \frac{\varepsilon}{6}\\
        & \leq   \sqrt{\frac{c_1\log \left( \frac{S A}{(1-\gamma)\delta \varepsilon}\right)}{n}} \sqrt{\frac{2}{1-\gamma}}  \sqrt{5 \frac{H}{1-\gamma} }  + c_1 \gamma \frac{\log \left( \frac{S A}{(1-\gamma)\delta \varepsilon}\right)}{(1-\gamma)n} \infnorm{V_\gamma^{\pistar_\gamma}}+ \frac{\varepsilon}{6} \\
        & \leq  \sqrt{\frac{c_1 \log \left( \frac{S A}{(1-\gamma)\delta \varepsilon}\right)}{n}}   \sqrt{10 \frac{H}{(1-\gamma)^2} }  + c_1  \frac{\log \left( \frac{S A}{(1-\gamma)\delta \varepsilon}\right)}{(1-\gamma)^2n} + \frac{\varepsilon}{6}
    \end{align*}
    where in the last inequality we used the facts that $\infnorm{V_\gamma^{\pistar_\gamma}} \leq \frac{1}{1-\gamma}$ and $\gamma \leq 1$. Now if we assume $n \geq  360 c_1\frac{H}{(1-\gamma)^2\varepsilon^2}\log \left( \frac{S A}{(1-\gamma)\delta \varepsilon}\right)$, we have 
    \begin{align*}
        \infnorm{\Vhat_{\gamma, \pert}^{\pistar_\gamma} - V_\gamma^{\pistar_\gamma}} 
        & \leq   \sqrt{\frac{c_1 \log \left( \frac{S A}{(1-\gamma)\delta \varepsilon}\right)}{n}} \sqrt{10 \frac{H}{(1-\gamma)^2} }  + c_1  \frac{\log \left( \frac{S A}{(1-\gamma)\delta \varepsilon}\right)}{(1-\gamma)^2n} + \frac{\varepsilon}{6}\\
        & \leq \frac{1}{6}\sqrt{\varepsilon^2} + \frac{1}{6}\frac{\varepsilon^2}{H}+ \frac{\varepsilon}{6}\\
        & \leq \varepsilon/2
    \end{align*}
    due to the fact that $\varepsilon \leq H$.
    
    Next, to bound $\infnorm{\Vhat_{\gamma, \pert}^{\pihstar_{\gamma, \pert}} - V_\gamma^{\pihstar_{\gamma, \pert}}}$, starting from inequality~\eqref{eq:pihatstar_error} in Lemma \ref{lem:DMDP_error_bounds} and then analogously applying Lemma \ref{lem:var_params_relationship} and then Lemma \ref{lem:pihatstar_var_bound}, we obtain
    \begin{align*}
        &\quad \infnorm{\Vhat_{\gamma, \pert}^{\pihstar_{\gamma, \pert}} - V_\gamma^{\pihstar_{\gamma, \pert}}} \\
        & \leq  \gamma \sqrt{\frac{c_1\log \left( \frac{S A}{(1-\gamma)\delta \varepsilon}\right)}{n}} \infnorm{(I - \gamma P_{\pihstar_{\gamma, \pert}})^{-1} \sqrt{\Var_{P_{\pihstar_{\gamma, \pert}}} \left[V_{\gamma, \pert}^{\pihstar_{\gamma, \pert}} \right]}} + c_1 \gamma\frac{\log \left( \frac{S A}{(1-\gamma)\delta \varepsilon}\right)}{(1-\gamma)n} \infnorm{V_{\gamma, \pert}^{\pihstar_{\gamma, \pert}}} + \frac{\varepsilon}{6}\\
        & \leq \sqrt{\frac{c_1\log \left( \frac{S A}{(1-\gamma)\delta \varepsilon}\right)}{n}}  \sqrt{\frac{2}{1-\gamma}}  \sqrt{\infnorm{\Var^{\pihstar_{\gamma, \pert}}\left[ \sum_{t=0}^{\infty} \gamma^t {\Rpert}_t  \right]} }+ c_1 \gamma\frac{\log \left( \frac{S A}{(1-\gamma)\delta \varepsilon}\right)}{(1-\gamma)n} \infnorm{V_{\gamma, \pert}^{\pihstar_{\gamma, \pert}}} + \frac{\varepsilon}{6}\\
        & \leq   \sqrt{\frac{c_1\log \left( \frac{S A}{(1-\gamma)\delta \varepsilon}\right)}{n}}  \sqrt{\frac{2}{1-\gamma}}  \sqrt{   15\frac{H^2 + \infnorm{V_\gamma^{\pihstar_{\gamma, \pert}} - \Vhat_{\gamma, \pert}^{\pihstar_{\gamma, \pert}}}^2 + \infnorm{V_\gamma^{\pistar_\gamma} - \Vhat_{\gamma, \pert}^{\pistar_\gamma}}^2}{H(1-\gamma)}  }+ c_1 \gamma\frac{\log \left( \frac{S A}{(1-\gamma)\delta \varepsilon}\right)}{(1-\gamma)n} \infnorm{V_{\gamma, \pert}^{\pihstar_{\gamma, \pert}}}+ \frac{\varepsilon}{6}.
    \end{align*}
    Combining with the fact from above that $\infnorm{\Vhat_{\gamma, \pert}^{\pistar_\gamma} - V_\gamma^{\pistar_\gamma}}  \leq \frac{H}{2}$, as well as the facts that $\infnorm{V_{\gamma, \pert}^{\pihstar_{\gamma, \pert}}} \leq \frac{1}{1-\gamma}$, $\gamma\leq 1$, and $\sqrt{a+b} \leq \sqrt{a}+\sqrt{b}$, we have
    \begin{align*}
        \infnorm{\Vhat_{\gamma, \pert}^{\pihstar_{\gamma, \pert}} - V_\gamma^{\pihstar_{\gamma, \pert}}} &  \leq   \sqrt{\frac{c_1\log \left( \frac{S A}{(1-\gamma)\delta \varepsilon}\right)}{n}}  \sqrt{\frac{2}{1-\gamma}}  \sqrt{   15\frac{\frac{5}{4}H^2 + \infnorm{V_\gamma^{\pihstar_{\gamma, \pert}} - \Vhat_{\gamma, \pert}^{\pihstar_{\gamma, \pert}}}^2 }{H(1-\gamma)}  }+ c_1 \frac{\log \left( \frac{S A}{(1-\gamma)\delta \varepsilon}\right)}{(1-\gamma)^2n} + \frac{\varepsilon}{6}\\
        & \leq    \sqrt{\frac{c_1\log \left( \frac{S A}{(1-\gamma)\delta \varepsilon}\right)}{n}}   \sqrt{\frac{30}{H(1-\gamma)^2}} \left(  \sqrt{   \frac{5}{4}H^2} + \sqrt{\infnorm{V_\gamma^{\pihstar_{\gamma, \pert}} - \Vhat_{\gamma, \pert}^{\pihstar_{\gamma, \pert}}}^2  } \right)+ c_1 \frac{\log \left( \frac{S A}{(1-\gamma)\delta \varepsilon}\right)}{(1-\gamma)^2n} + \frac{\varepsilon}{6}\\
        &=   \sqrt{\frac{c_1\log \left( \frac{S A}{(1-\gamma)\delta \varepsilon}\right)}{n}}   \sqrt{\frac{30}{H(1-\gamma)^2}} \left(  \sqrt{   \frac{5}{4}}H + \infnorm{V_\gamma^{\pihstar_{\gamma, \pert}} - \Vhat_{\gamma, \pert}^{\pihstar_{\gamma, \pert}}} \right)+ c_1 \frac{\log \left( \frac{S A}{(1-\gamma)\delta \varepsilon}\right)}{(1-\gamma)^2n}+ \frac{\varepsilon}{6}.
    \end{align*}
    Rearranging terms gives
    \begin{align*}
        &\left(1 -  \sqrt{\frac{c_1\log \left( \frac{S A}{(1-\gamma)\delta \varepsilon}\right)}{n}} \sqrt{\frac{30}{H(1-\gamma)^2}} \right) \infnorm{\Vhat_{\gamma, \pert}^{\pihstar_{\gamma, \pert}} - V_\gamma^{\pihstar_{\gamma, \pert}}} \\
        & \leq   \sqrt{\frac{c_1\log \left( \frac{S A}{(1-\gamma)\delta \varepsilon}\right)}{n}}   \sqrt{\frac{75H/2}{(1-\gamma)^2}} + c_1 \frac{\log \left( \frac{S A}{(1-\gamma)\delta \varepsilon}\right)}{(1-\gamma)^2n} + \frac{\varepsilon}{6}.
    \end{align*}
    Assuming $n \geq 120 c_1 \frac{H}{(1-\gamma)^2\varepsilon^2}\log \left( \frac{S A}{(1-\gamma)\delta \varepsilon}\right)$, we have
    \begin{align*}
        1 -   \sqrt{\frac{c_1\log \left( \frac{S A}{(1-\gamma)\delta \varepsilon}\right)}{n}} \sqrt{\frac{30}{H(1-\gamma)^2}} & \geq 1 - \frac{1}{2} \sqrt{ \frac{\varepsilon^2 (1-\gamma^2)}{H} \frac{1}{H(1-\gamma)^2}} = 1- \frac{1}{2}\frac{\varepsilon}{H} \geq \frac{1}{2}
    \end{align*}
    since $\varepsilon \leq H$. Also assuming $n \geq (75/2)\cdot 24^2 c_1\frac{H}{(1-\gamma)^2\varepsilon^2}\log \left( \frac{S A}{(1-\gamma)\delta \varepsilon}\right)$ we have similarly to before that
    \begin{align*}
            \sqrt{\frac{c_1\log \left( \frac{S A}{(1-\gamma)\delta \varepsilon}\right)}{n}}   \sqrt{\frac{75H/2}{(1-\gamma)^2}} + c_1 \frac{\log \left( \frac{S A}{(1-\gamma)\delta \varepsilon}\right)}{(1-\gamma)^2n} + \frac{\varepsilon}{6} &\leq \frac{1}{24}\sqrt{\frac{(1-\gamma)^2 \varepsilon^2}{H} \frac{H}{(1-\gamma)^2}}+ \frac{1}{24}\frac{(1-\gamma)^2\varepsilon^2}{H}\frac{1}{(1-\gamma)^2} +\frac{\varepsilon}{6}\\
        & \leq \frac{\varepsilon}{24} + \frac{\varepsilon}{24} + \frac{\varepsilon}{6}= \frac{\varepsilon}{4}.
    \end{align*}
    Combining these two calculations, we have $\frac{1}{2}\infnorm{\Vhat_{\gamma, \pert}^{\pihstar_{\gamma, \pert}} - V_\gamma^{\pihstar_{\gamma, \pert}}} \leq \frac{\varepsilon}{4}$, so $\infnorm{\Vhat_{\gamma, \pert}^{\pihstar_{\gamma, \pert}} - V_\gamma^{\pihstar_{\gamma, \pert}}} \leq \frac{\varepsilon}{2}$ as desired.

    Since we have established that $\infnorm{\Vhat_{\gamma, \pert}^{\pistar_\gamma} - V_\gamma^{\pistar_\gamma}}, \infnorm{\Vhat_{\gamma, \pert}^{\pihstar_{\gamma, \pert}} - V_\gamma^{\pihstar_{\gamma, \pert}}} \leq \frac{\varepsilon}{2}$, since also $\Vhat_{\gamma, \pert}^{\pihstar_{\gamma, \pert}} \geq\Vhat_{\gamma, \pert}^{\pistar_\gamma}$, we can conclude that
    \begin{align*}
        V_\gamma^{\pistar_\gamma} - V_\gamma^{\pihstar_{\gamma, \pert}}  \leq \infnorm{\Vhat_{\gamma, \pert}^{\pistar_\gamma} - V_\gamma^{\pistar_\gamma}}\one + \infnorm{\Vhat_{\gamma, \pert}^{\pihstar_{\gamma, \pert}} - V_\gamma^{\pihstar_{\gamma, \pert}}} \one \leq \varepsilon \one,
    \end{align*}
    that is that $\pihstar_{\gamma, \pert}$ is $\varepsilon$-optimal for the discounted MDP $(P, r ,\gamma)$. 
    
    We finally note that all our requirements on the size of $n$ can be satisfied by requiring
    \begin{align*}
        n &\geq  \max\left\{\frac{c_2}{1-\gamma}, \frac{360c_1 H }{(1-\gamma)^2 \varepsilon^2}, \frac{(75/2)24^2c_1 H }{(1-\gamma)^2 \varepsilon^2} \right\}\log \left( \frac{S A}{(1-\gamma)\delta \varepsilon}\right) \\
        & \geq \max\left\{\frac{c_2H}{(1-\gamma)^2 \varepsilon^2}, \frac{360c_1 H }{(1-\gamma)^2 \varepsilon^2}, \frac{(75/2)24^2c_1 H }{(1-\gamma)^2 \varepsilon^2} \right\}\log \left( \frac{S A}{(1-\gamma)\delta \varepsilon}\right) \\
        & := C_2 \frac{H}{(1-\gamma)^2 \varepsilon^2} \log \left( \frac{S A}{(1-\gamma)\delta \varepsilon}\right)
    \end{align*}
    where we used that $\frac{H}{(1-\gamma)^2\varepsilon^2} \geq \frac{H^2}{(1-\gamma)\varepsilon^2} \geq \frac{1}{1-\gamma}$ (since $\frac{1}{1-\gamma} \geq H$ and $H \geq \varepsilon$).
\end{proof}

Now we can use Theorem \ref{thm:DMDP_bound} to complete our proof of the main theorem \ref{thm:main_theorem}.
\begin{proof}[Proof of Theorem \ref{thm:main_theorem}]
    Applying Lemma \ref{lem:DMDP_reduction} (with error parameter $\frac{\varepsilon}{12}$ since we have chosen $\gammared = 1 - \frac{\varepsilon/12}{H} $), we have that
    \begin{align*}
        \rho^{\star} - \rho^{\pihstar_\gamma} \leq \left(8 + 3\frac{H}{H} \right)\frac{\varepsilon}{12} \leq \varepsilon \one
    \end{align*}
    as desired.
\end{proof}

\subsection{Proof of Proposition~\ref{thm:optimal_tmix_bound}}

Finally we prove Proposition \ref{thm:optimal_tmix_bound}.
\begin{proof}[Proof of Proposition \ref{thm:optimal_tmix_bound}]
    First we show the inequality $H \leq 8\tstar$. If $\Pi^\star$ is empty or if it only contains $\pi$ such that $\tau_\pi = \infty$, then $\tstar = \infty$ and the inequality trivially holds, so now we can consider the case that $\Pi^\star$ is nonempty and $\tstar < \infty$. Let $\pi \in \Pi^\star$ be a policy attaining the supremum $\tau_\pi = \tstar$ (the supremum in the definition of $\tstar$ must always be attained by some policy since the set of all deterministic policies is finite). Since $\pi \in \Pi^\star$ and $\tau_\pi < \infty$, the Markov chain with transition matrix $P_\pi$ is aperiodic and has a unique stationary distribution $\nu_\pi$. Aperiodicity guarantees \cite{puterman_markov_2014} that for all $s \in \S$,
    \begin{align*}
        h^\star(s) = h^{\pi}(s) = \lim_{T \to \infty} \E^\pi_s \left[\sum_{t=0}^{T-1}R_t - T \rho^\pi(s) \right] =  \lim_{T \to \infty}\sum_{t=0}^{T-1} e_s^\top \left(P_\pi \right)^t r_\pi - T \rho^\pi(s).
    \end{align*}
    Furthermore by optimality of $\pi$ we have $\rho^\pi(s) = \rho^\star$ (that is, $\rho^\pi$ is a constant vector), so
    \begin{align*}
        \spannorm{h^\star} &= \spannorm{ \lim_{T \to \infty} \sum_{t=0}^{T-1} \left(P_\pi \right)^t r_\pi - T \rho^\pi} \\
        &=  \lim_{T \to \infty} \spannorm{\sum_{t=0}^{T-1} \left(P_\pi \right)^t r_\pi - T \rho^\pi} \\
        &=  \lim_{T \to \infty} \spannorm{\sum_{t=0}^{T-1} \left(P_\pi \right)^t r_\pi}.
    \end{align*}
    
    The following calculation is similar to that of \cite[Lemma 9]{wang_near_2022}. Define
    \[
    d(t) = \max_{s \in \S} \frac{1}{2}\onenorm{e_s^\top \left(P_\pi \right)^t - \nu_\pi^\top}.
    \]
    It is shown in \cite[Section 4.5]{levin_markov_2017} that $d(k \tau_\pi) \leq 2^{-k}$ for all integers $k \geq 0$. Now we can calculate that for any $T$,
    \begin{align*}
        \spannorm{\sum_{t=0}^{T-1} \left(P_\pi \right)^t r_\pi} & \leq \sum_{t=0}^{T-1} \spannorm{ \left(P_\pi \right)^t r_\pi} \\
        & = \sum_{t=0}^{T-1} \spannorm{ \left(\left(P_\pi \right)^t - \one \nu_\pi^\top \right)  r_\pi} \\
        & \leq \sum_{t=0}^{T-1} 2\max_{s \in \S} \left|e_s^\top\left(\left(P_\pi \right)^t - \one \nu_\pi^\top \right)  r_\pi \right|\\
        & \leq \sum_{t=0}^{T-1} 2\max_{s \in \S} \onenorm{e_s^\top\left(\left(P_\pi \right)^t - \one \nu_\pi^\top \right)} \infnorm{ r_\pi} \\
        &\leq \sum_{t=0}^{T-1} 4 d(t) \\
        &\leq \sum_{t=0}^{\infty} 4 d(t) \\
        &\leq \sum_{k=0}^{\infty} 4 d \tau_\pi (k \tau_\pi) \\
        &\leq 4 \tau_\pi \sum_{k=0}^{\infty} 2^{-k} \\
        &= 8 \tau_\pi,
    \end{align*}
where in the second last inequality step we used the fact that $d(t)$ is decreasing in $t$ and grouped consecutive terms into groups of size $\tau_\pi$. Therefore we conclude that 
\[\spannorm{h^\star} = \lim_{T \to \infty} \spannorm{\sum_{t=0}^{T-1} \left(P_\pi \right)^t r_\pi} \leq 8\tau_\pi = 8\tstar,
\] 
thereby proving the first part of Proposition~\ref{thm:optimal_tmix_bound}.

For the second part of the proposition, since Theorems \ref{thm:DMDP_bound} and \ref{thm:main_theorem} hold as long as $H$ is any upper bound of $\spannorm{h^\star}$, they hold with $8 \tstar$ in place of $H$.
\end{proof}

\section{Conclusion}
In this paper we resolved the sample complexity of learning an $\varepsilon$-optimal policy in an average reward MDP in terms of the span of the optimal bias function $\spannorm{h^\star}$, removing the assumption of uniformly bounded mixing times and matching the minimax lower bound. We refined the analysis behind the technique of reducing to a discounted MDP and obtained sample complexity bounds for discounted MDPs in terms of $\spannorm{h^\star}$, which circumvent the minimax lower bound when the MDP is weakly communicating. We believe these results shed greater light on the relationship between the discounted and average reward settings, and we hope that our technical developments can be reused in future work to broaden our understanding of average-reward RL.

\printbibliography

\end{document}